\documentclass{article}

 \usepackage[preprint]{neurips_2026}
\usepackage{amsmath, amssymb, amsthm}
\usepackage{multirow}
\usepackage{wrapfig}
\usepackage{xcolor}  
\usepackage{tcolorbox}
\tcbset{
  colback=blue!8!white,
  colframe=blue!8!white, 
  boxrule=0pt,
  arc=4pt,
}
\usepackage{tabularx}
\usepackage{xcolor}
\usepackage{tcolorbox}
\usepackage{wrapfig}
\definecolor{flexokilavander}{HTML}{A699D0}
\colorlet{flexokilavanderlight}{flexokilavander!18!white}
\usepackage[utf8]{inputenc} 
\usepackage[T1]{fontenc}    
\usepackage{hyperref}       
\usepackage{url}            
\usepackage{booktabs}       
\usepackage{amsfonts}       
\usepackage{nicefrac}       
\usepackage{microtype}      
\theoremstyle{plain}
\newtheorem{theorem}{Theorem}[section]
\newtheorem{lemma}[theorem]{Lemma}

\newtheorem{proposition}[theorem]{Proposition}

\newtheorem{definition}[theorem]{Definition}

\usepackage{enumitem}

\definecolor{deepblue}{RGB}{20,50,120}
\definecolor{softblue}{RGB}{220,230,255}
\definecolor{deepgreen}{RGB}{10,100,50}
\definecolor{softgreen}{RGB}{215,240,225}
\definecolor{darkgray}{RGB}{60,60,60}
\definecolor{manifoldred}{RGB}{180,30,30}
\definecolor{tanblue}{RGB}{30,80,200}

\newcommand{\M}{\mathcal{M}}

\newcommand{\Exp}{\mathrm{Exp}}
\newcommand{\Log}{\mathrm{Log}}

\usepackage[table]{xcolor}

\definecolor{bestblue}{RGB}{189,215,238}
\definecolor{secondblue}{RGB}{221,235,247}
\definecolor{thirdgrayblue}{RGB}{239,243,246}
\definecolor{baselinegray}{RGB}{245,245,245}
\definecolor{driftblue}{RGB}{232,242,255}

\newcommand{\best}[1]{\cellcolor{bestblue}#1}
\newcommand{\second}[1]{\cellcolor{secondblue}#1}
\newcommand{\third}[1]{\cellcolor{thirdgrayblue}#1}

\newtcolorbox{definitionbox}[1]{
  paperbox,
  colback=gray!5,
  colframe=gray!45,
  title={Definition: #1}
}

\newtcolorbox{resultbox}[1]{
  paperbox,
  colback=blue!3,
  colframe=blue!45!black,
  title={#1}
}

\newtcolorbox{takeawaybox}[1]{
  paperbox,
  colback=orange!4,
  colframe=orange!55!black,
  title={#1}
}

\newcommand{\std}[1]{{\scriptscriptstyle\pm#1}}
\usepackage{comment}
\usepackage{arydshln}
\usepackage{float}
\usepackage{graphicx}
\usepackage{hyperref}
\title{Kernel-Gradient Drifting Models}

\author{%
  Maria Esteban-Casadevall \thanks{Equal contribution} \, \thanks{ \texttt{m.estebancasadevall@uva.nl}}\\
  AMLab\\
  University of Amsterdam\\
  \And
  Jorge Carrasco-Pollo$^*$  \\
  University of Amsterdam \\
   \AND
   Max Welling \\
   CuspAI \\
  UvA-Bosch Delta Lab, AMLab \\
  University of Amsterdam \\
   \And
   Jan-Willem van de Meent\\
  UvA-Bosch Delta Lab, AMLab \\
  University of Amsterdam \\
     \And
    Erik J. Bekkers \\
    AMLab \\
    University of Amsterdam \\
   \And
   Floor Eijkelboom \\
  UvA-Bosch Delta Lab, AMLab \\
  University of Amsterdam \\
}

\begin{document}

\maketitle
\begin{abstract}
We propose kernel-gradient drifting, a one-step generative modeling framework that replaces the fixed Euclidean displacement direction in drifting models with directions induced by the kernel itself. Standard drifting is attractive because it enables fast, high-quality generation without distilling a large pretrained diffusion model, but its theory is currently understood mainly for Gaussian kernels, where the drift coincides with smoothed score matching and is identifiable. Our gradient-based reformulation exposes this score-based structure for general kernels: the resulting drift is the score difference between kernel-smoothed data and model distributions, yielding identifiability for characteristic kernels and a smoothed-KL descent interpretation of the smoothed drifting dynamics. Since kernel gradients are intrinsic tangent vectors, the same construction extends naturally to Riemannian manifolds and to discrete data via the Fisher-Rao geometry of the probability simplex. Across spherical geospatial data, promoter DNA and molecule generation, kernel-gradient drifting enables state-of-the-art one-step generation beyond the Euclidean setting without distillation. Code is available \href{https://github.com/jcarrascopollo/kgd_thesis}{here}.
\end{abstract}

\section{Introduction}
Many generative models rely on transport-based methods that map a simple prior to a complex data distribution \citep{Lipman_2023, Ho2020DenoisingDiffusion, Song2021ScoreBased}. These methods are empirically strong and well understood, but sampling typically requires solving an SDE or ODE through many sequential updates. This has motivated a growing line of work on one-step or few-step generation, either by distilling a pretrained model \citep{Salimans2022ProgressiveDistillation, Luo2023DiffInstruct, Yin_2024_DMD} or by training such generators directly through flow-map, consistency-style, or velocity-matching objectives \citep{boffi2025flowmap,boffi2025howto, Davis_2026_GFM, Roos_2026_CatFlowMaps, Geng2025MeanFlows, Geng2025ImprovedMeanFlows, Zhou2025TerminalVelocityMatching}. Drifting models~\citep{Deng_2026_Drift} instead train a one-step generator directly by moving model samples along a kernel-induced attraction and repulsion field. The kernel measures similarity between samples and determines the strength of their pairwise interactions, making it an important design choice for the model. In this way, drifting shifts the refinement that would normally occur at sampling time into the training procedure.


Current theory for drifting is best understood when working with Gaussian kernels. However, other non-Gaussian kernels, such as Laplace, seem to perform better empirically, making it important to understand drifting beyond the Gaussian case. In the Gaussian case, the original kernel-weighted displacement field coincides with a smoothed score-matching direction, which gives a clean \textit{identifiability} story: at the population level, vanishing drift implies that the model distribution matches the data distribution~\citep{Lai_2026_Unified,Turan_2026_Drifting}, an essential property for the model to be sound. At the same time, this result also exposes a point of tension in the original formulation. Drifting combines two choices: the kernel \(k(x,y)\), which determines how samples are weighted, and the Euclidean displacement \(y-x\), which determines the direction of motion. For Gaussian kernels this displacement direction is natural, as it points in the direction of steepest increase in similarity. For other kernels, however, the same displacement direction doesn't necessarily point in the direction of highest increase in similarity and, in general, doesn't ensure identifiability.  Consequently, it is  unclear when the model converges to the data distribution, what objective the drifting dynamics are optimizing, or how the method should be formulated on non-Euclidean domains.

\textbf{Contributions.} We address this by introducing \emph{kernel-gradient drifting}. Instead of using the kernel only to weight Euclidean displacements, we let the kernel itself define the direction of motion through its gradient. This recovers ordinary Gaussian drifting as a special case, while providing a more general and coherent formulation of drifting beyond Euclidean Gaussian kernels. The key consequence is that drifting is no longer tied to a particular displacement rule or ambient Euclidean geometry, but becomes a kernel-defined gradient flow. In particular, this re-formulation  extends the score-based interpretation to general kernels: the resulting drift is exactly the difference between the scores of kernel-smoothed data and model distributions. This gives identifiability whenever the smoothing kernel is characteristic, and yields a variational interpretation of the smoothed drifting dynamics as descent of a smoothed Kullback--Leibler divergence. Because kernel gradients are intrinsic tangent vectors, the same construction also extends naturally to Riemannian manifolds, and allows for discrete data generation through the Fisher--Rao geometry of the probability simplex. 

We evaluate our ideas across three complementary settings \footnote{Code available at \url{https://github.com/jcarrascopollo/kgd_thesis}.}. Controlled synthetic experiments test whether the gradient drift direction improves performance. The Earth event modeling on \(\mathbb{S}^2\) tests how different kernel choices behave on real spherical data, and promoter DNA generation and molecule generation (through SMILES strings) test whether the simplex construction gives a practical route to one-step discrete generation. Across these settings, kernel-gradient drifting provides a principled extension of drifting beyond the Gaussian Euclidean case, while remaining teacher-free and competitive with one-step flow and distillation-based baselines.
\begin{figure}[ht!]
\centering
\includegraphics[width=1.0\linewidth]{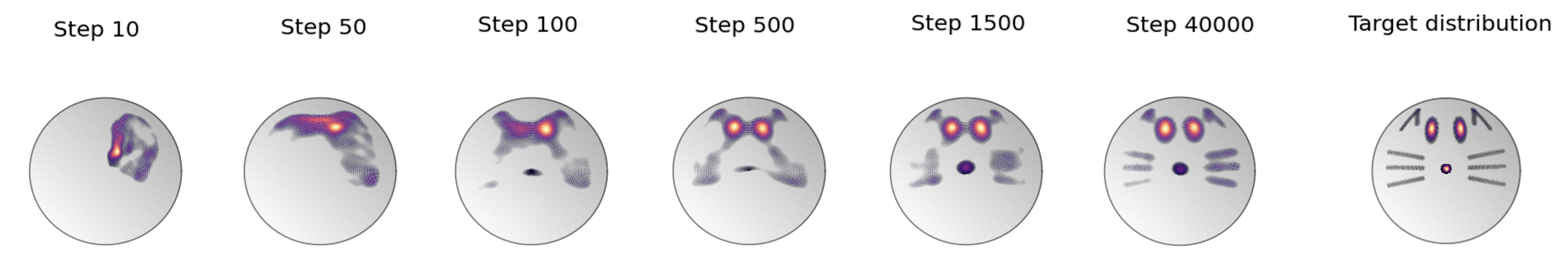}
\caption{
Riemannian kernel-gradient drifting on \(\mathbb{S}^2\). Samples are transported
along tangent vectors and mapped back to the sphere with the exponential map,
allowing drifting to respect the geometry of the data support.
}
\label{fig:spherical_cat}
\end{figure}

\section{Drifting and the Gaussian limitation}
\label{sec:background}

\paragraph{Drifting models.} Drifting models~\cite{Deng_2026_Drift} train one-step generators by moving the iterative refinement usually performed at sampling time into the training procedure. Let \(p\) be a target distribution on \(\mathbb{R}^d\), and let
\(f_\theta : \mathbb{R}^k \to \mathbb{R}^d\) be a generator with pushforward distribution
\(q_\theta = (f_\theta)_{\#}\mathcal{N}(0,I)\). Given a positive normalized kernel
\(k : \mathbb{R}^d \times \mathbb{R}^d \to \mathbb{R}_{>0}\), drifting defines a discrepancy field
\[
V_{p,q_\theta}(x) = V_p^{+}(x) - V_{q_\theta}^{-}(x),
\]
where
\begin{equation}\label{eq:drifting_field}
V_p^{+}(x)
=
\underbrace{\frac{\mathbb{E}_{y \sim p}\!\left[k(x,y)(y-x)\right]}
     {\mathbb{E}_{y \sim p}\!\left[k(x,y)\right]}}_{\text{attraction to data}},
\qquad
V_{q_\theta}^{-}(x)
=
\underbrace{\frac{\mathbb{E}_{x' \sim q_\theta}\!\left[k(x,x')(x'-x)\right]}
     {\mathbb{E}_{x' \sim q_\theta}\!\left[k(x,x')\right]}}_{\text{self-repulsion}}.
\end{equation}

Thus, \(V_{p,q_\theta}\) has an attractive--repulsive structure: it \textit{attracts} samples towards regions where the data distribution is locally better represented than the current model distribution, while \textit{repelling} model samples from each other to ensure coverage. During training, a noise sample \(\epsilon \sim \mathcal{N}(0,I)\) is mapped to
\(x = f_\theta(\epsilon)\), and transported to
\begin{equation}\label{eq:displacement}
\widetilde{x}
=
x + \eta V_{p,q_\theta}(x),
\end{equation}
where $\eta \in \mathbb{R}$ is a step-size. The generator \(f_\theta\) is trained via a stop-gradient loss, whose \textit{value} is
\begin{equation}
\label{eq:drifting_loss}
\mathcal{L}_{\mathrm{drift}}(\theta)
=
\mathbb{E}_{\epsilon \sim \mathcal{N}(0,I)}
\left\|
f_\theta(\epsilon)
-
\mathrm{sg}\!\left(\widetilde{x}\right)
\right\|_2^2=\eta^2
\mathbb{E}_{x \sim q_\theta}
\left\|
V_{p,q_\theta}(x)
\right\|_2^2.
\end{equation}
The transported samples are therefore treated as a frozen target, and the next generator trains on this transported sample cloud. \citep{Turan_2026_Drifting} shows that applying stop-gradient is algorithmically necessary for ensuring convergence of the method.

\paragraph{Identifiability.}
For drifting to define a meaningful population objective, vanishing drift should imply that the model distribution has matched the data distribution. Otherwise, the generator could reach a fixed point even when \(q_\theta \neq p\). We refer to this property as \emph{identifiability}: the drift operator \(V_{p,q_\theta}\) is identifiable if
\begin{equation}
\label{eq:identifiability}
V_{p,q_\theta}(x)=0 \;\; \text{for all } x
\qquad \Longrightarrow \qquad
p=q_\theta.
\end{equation}


\paragraph{The Gaussian case.}
Concurrent work has shown that drifting has a clean score-based interpretation for Gaussian kernels~\cite{Lai_2026_Unified,Turan_2026_Drifting}. Consider
\[
k_\tau(x,y)
=
\frac{1}{(2\pi\tau^2)^{d/2}}\exp\!\left(
-\frac{\|x-y\|_2^2}{2\tau^2}
\right).
\]
for a temperature value $\tau > 0$. For this kernel, 

\[
\nabla_x k_\tau(x,y) =
\frac{1}{\tau^2} k_\tau(x,y)(y-x).\]

Hence the displacement direction used in \eqref{eq:drifting_field} is proportional to the gradient of the kernel itself. If we define the Gaussian convolved density
\[
\widehat{p}_\tau(x) = \mathbb{E}_{y\sim p}\!\left[k_\tau(x,y)\right],  \qquad
 \widehat{q}_{\theta, \tau}(x)
=
\mathbb{E}_{x'\sim q_\theta}\!\left[k_\tau(x,x')\right], 
\]
then
\[
V_{p}^+(x) = \tau^2 \nabla_x \log \widehat{p}_\tau(x)\qquad
 V_{q_\theta}^-(x)= \tau^2 \nabla_x \log \widehat{q}_{\theta, \tau}(x) .
\]
Consequently, Gaussian drifting satisfies
\[
V_{p,q_\theta}(x)
=
\tau^2
\left(
\nabla_x \log \widehat{p}_\tau(x)
-
\nabla_x \log \widehat{q}_{\theta,\tau}(x)
\right)
\]  and the training objective reduces to one-step smoothed score matching in reverse-Fisher form. Vanishing drift implies equality of the smoothed scores, and injectivity of the Gaussian kernel yields \(p=q_\theta\). This gives Gaussian drifting an identifiable population fixed point. This formulation also reduces the drifting objective to the \textit{score-difference} setting described in \citep{weber2023score}, connecting kernel drifting methods to previous generative models literature. 

\section{Kernel-gradient drifting}
\label{sec:kernel_gradient_drifting}

The previous section shows that Gaussian drifting is special because the kernel-weighted displacement direction coincides with a smoothed score direction. This suggests that the natural object is not the Euclidean displacement
\(y-x\) itself, but the direction induced by the kernel. In this section, we use this observation to define \textit{kernel-gradient drifting}.

\subsection{What goes wrong in the original formulation?}
The Gaussian result also exposes the limitation of the original formulation. Equation~\eqref{eq:drifting_field} combines two separate design choices: the kernel \(k(x,y)\), which determines how nearby samples are weighted, and the Euclidean displacement \(y-x\), which determines the direction of motion. For Gaussian kernels these two choices are compatible, because the displacement-weighted update is proportional to a kernel-gradient direction. As a consequence, the direction of the drift field is exactly the mean-shift direction of the smoothed density ratio
\(\hat p_\tau / \hat q_{\theta,\tau}\) and, thus, the model moves samples toward regions where the smoothed data distribution is locally under-represented relative to the smoothed model distribution. For a general kernel, however,
\[
k(x,y)(y-x)
\not\propto
\nabla_x k(x,y)
\]
and the induced direction no longer coincides with the mean-shift direction. Instead, it is biased toward distant neighbors, whose influence is amplified by the distance-dependent directional term (see Figure \ref{fig: Figure1} and Appendix~\ref{app: original formulation} for details). This perturbation becomes more pronounced at larger temperatures, since farther points contribute more strongly, and on manifolds with high curvature. It is precisely this perturbation that breaks the identifiability property for non-Gaussian kernels. This is especially relevant because non-Gaussian kernels, such as Laplace kernels, often perform well in practice.  \citet{Turan_2026_Drifting} attribute this behavior to the spectral properties of the Laplace kernel which make its Fourier modes decay polynomially rather than exponentially, unlike in the Gaussian case, which helps preserve finer details of the distribution.  


\subsection{The kernel-gradient drift field}

Let \(k : \mathbb{R}^d \times \mathbb{R}^d \to \mathbb{R}_{>0}\) be a positive normalized
kernel differentiable in its first argument. Define
the \textit{kernel-gradient drift}
\begin{equation}
\label{eq:kernel_gradient_component}
G_{p}^+(x)
=
\frac{
\mathbb{E}_{y \sim p}\!\left[\nabla_x k(x,y)\right]
}{
\mathbb{E}_{y \sim p}\!\left[k(x,y)\right]
}, \qquad G_{q_\theta}^-(x) = \frac{
\mathbb{E}_{x' \sim q_\theta}\!\left[\nabla_x k(x,x')\right]
}{
\mathbb{E}_{x' \sim q_\theta}\!\left[k(x,x')\right]}.
\end{equation}
We then define the kernel-gradient drift as
\begin{equation}
\label{eq:gradient_drifting}
V^\nabla_{p,q_\theta}(x)
=
G_{p}^+(x) - G_{q_\theta}^-(x)
=
\frac{
\mathbb{E}_{y \sim p}
\!\left[
\nabla_x k(x,y)
\right]
}{
\mathbb{E}_{y \sim p}
\!\left[
k(x,y)
\right]
}
-
\frac{
\mathbb{E}_{x' \sim q_\theta}
\!\left[
\nabla_x k(x,x')
\right]
}{
\mathbb{E}_{x' \sim q_\theta}
\!\left[
k(x,x')
\right]
}.
\end{equation}
When the kernel is Gaussian, this recovers the original drifting field up to the
constant temperature factor $\tau^2$. Training proceeds exactly as described in Section \ref{sec:background}.

\begin{figure}[ht!]
\centering
\includegraphics[width=0.95\linewidth]{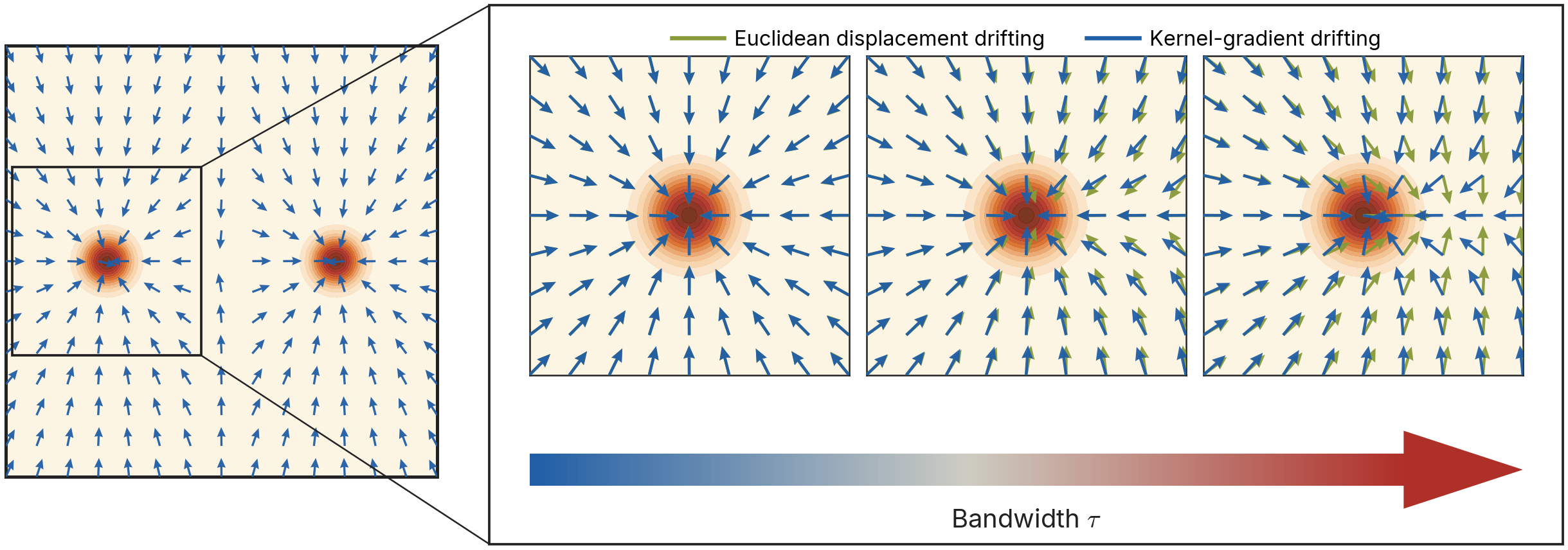}
\caption{Attraction drift direction induced by the Laplacian kernel. Under the original formulation, shown by the green arrows, the drift does not point towards the region of highest density ratio between $p$ and $q_\theta$. Instead, its direction is biased toward distant samples. This phenomenon becomes more pronounced as the temperature increases, since distant samples receive greater weight. At small temperatures, the two directions are approximately aligned.}
\label{fig: Figure1}
\end{figure}

\subsection{Smoothed score-ratio interpretation}

The main advantage of the gradient formulation is that it admits a score-based interpretation. Define the kernel-smoothed density
\[\label{eq:smoothed_density}
\widehat{p}_k(x)
= \int k(x,y)\,p(y)\,dy.\]

Note that when working in $\mathbb{R}^d$, this corresponds to applying a convolution operator to $p$ and $q_\theta$. This is not true on a curved manifold in general, as most  manifolds don't have a group structure. Then, under standard  regularity assumptions allowing differentiation under the
integral sign (see Appendix \ref{app: theoretical guarantees} for details), we have that $ G_{p}^+(x) = \nabla_x \log \widehat{p}_k(x).$ 
The same identity applied to negative samples coming from \(q_\theta\) leads to
\[
V^\nabla_{p,q_\theta}(x)
=
\nabla_x \log \widehat{p}_k(x)
-
\nabla_x \log \widehat{q}_{\theta,k}(x)
=
\nabla_x \log
\frac{\widehat{p}_k(x)}
     {\widehat{q}_{\theta,k}(x)}.
\]

This result is formalized in Proposition \ref{prop:score_ratio_matching}. See Appendix \ref{app: theoretical guarantees} for the formal statement, assumptions, and proof.

\definecolor{myboxcolor}{RGB}{252,245,224}

\begin{tcolorbox}[
    colback=flexokilavanderlight,
    colframe=flexokilavander,
    boxrule=1pt,
    arc=0pt
]
\begin{proposition}[Kernel-gradient drifting is smoothed score-ratio matching]
\label{prop:score_ratio_matching}
Let \(k\) be a positive, normalized kernel differentiable in its first argument, and assume that
differentiation may be exchanged with integration. Then the kernel-gradient drift
defined in \eqref{eq:gradient_drifting} satisfies
\[
V^\nabla_{p,q_\theta}(x)
=
\nabla_x \log
\frac{\widehat{p}_k(x)}
     {\widehat{q}_{\theta,k}(x)}.
\]
Consequently,
\[
\mathcal{L}^{\nabla}_\mathrm{drift}(\theta)
=
\eta^2
\mathbb{E}_{x\sim q_\theta}
\left\|
\nabla_x \log \widehat{p}_k(x)
-
\nabla_x \log \widehat{q}_{\theta,k}(x)
\right\|_2^2.
\]
\end{proposition}
\end{tcolorbox}

This connects kernel-gradient drifting models to the score-difference and mean-shift framework discussed in Section \ref{sec:background}, for any choice of kernel satisfying assumptions of Proposition \ref{prop:score_ratio_matching}.


\subsection{Identifiability with characteristic kernels}

The score-ratio identity gives a direct route to identifiability. If the drift
vanishes, then the smoothed data and model distributions have the same score. As we are working with normalized kernels, this implies that the smoothed densities are equal. The remaining question is whether equality of smoothed distributions implies equality of the original distributions. This is exactly the property of \textit{characteristic} kernels.

\begin{definition}[Characteristic kernel \cite{fukumizu2009characteristic}]
\label{def:characteristic_kernel}
A kernel \(k\) is characteristic on a class of probability distributions if the
smoothing map
\[
\pi
\mapsto
\widehat{\pi}_k,
\qquad
\widehat{\pi}_k(x)
=
\int k(x,y)\,\pi(y)\,dy,
\]
is injective. Equivalently, $
\widehat{p}_k(x)=\widehat{q}_k(x)
\quad \text{for all } x
\qquad
\Longrightarrow
\qquad
p=q.$
\end{definition}

\begin{tcolorbox}[
    colback=flexokilavanderlight,
    colframe=flexokilavander,
    boxrule=1pt,
    arc=0pt
]
\begin{proposition}[Identifiability]
\label{prop:identifiability}
Assume that the kernel \(k\) is characteristic and satisfies the assumptions from Proposition \ref{prop:score_ratio_matching}. Then
\[
V^\nabla_{p,q_\theta}(x)=0
\quad \text{for all } x \qquad
\Longrightarrow
\qquad p=q_\theta.\]
\end{proposition}
\end{tcolorbox}
Refer to Appendix \ref{app: theoretical guarantees} for the details on the proof. This result clarifies why an \textit{anti-symmetric} drift  is
important. If one considers an imbalanced field of the form $c\,G_{p}^+(x)-d\,G_{q}^-(x)$ for $ c\neq d,$  then vanishing drift would imply $ c\,\nabla \log \widehat{p}_k(x) = d\,\nabla \log \widehat{q}_k(x),$
which does not imply $\widehat{p}_k = \widehat{q}_k$ when $c\neq d$. Thus, the balanced
difference in \eqref{eq:gradient_drifting} is not merely a convention, but a necessary condition for distributional matching. This provides further insights into why \citet{Deng_2026_Drift} observed a catastrophic behavior when not working with anti-symmetric kernels. 

\newpage
\subsection{Smoothed-KL descent}
The score-ratio identity also gives an optimization interpretation of the drifting dynamics. 
\begin{tcolorbox}[
    colback=flexokilavanderlight,
    colframe=flexokilavander,
    boxrule=1pt,
    arc=0pt
]
\begin{proposition}[Smoothed-KL descent]
\label{prop:smoothed_kl_descent} Consider the drift update in the \textit{smoothed} variable \[z_{i+1}=z_i + \eta V_{p_\tau,q_{\theta,\tau}}(z_i),\] where  $z\sim \hat q_{\theta,k}.$  Then, the gradient drift direction gives the steepest infinitesimal decrease of the KL-divergence between the smoothed distributions $\hat p_k (z)$ and $\hat q_{\theta,k} (z)$. This corresponds to the Wasserstein gradient flow of the functional $D_{\mathrm{KL}}(\hat q_{\theta,k} \,\|\, \hat p_k)$.   \end{proposition}
\end{tcolorbox}

The proof follows from \cite{weber2023score} and is included in Appendix \ref{app: theoretical guarantees} for completeness. This result should be read together with Proposition~\ref{prop:identifiability}.
The gradient-drifting infinitesimally descends a smoothed distributional objective, and characteristic kernels ensure that matching the smoothed distributions identifies the original distributions.  

It is important, however, to distinguish this smoothed-law interpretation from
the original particle dynamics of \citep{Deng_2026_Drift}. Although the original updates act on unsmoothed samples, adding noise to samples is a common construction in generative models to improve training stability \citep{sonderby2017amortised, mescheder2018which}, as this mitigates the infinite KL divergence that can arise when the data and model distributions have non-overlapping support. Thus, Proposition~\ref{prop:smoothed_kl_descent} provides a useful interpretation of
drifting as descent on a noise-regularized distributional objective. A full
analysis of its impact on the training stability of drifting models is left as
future work.

\section{Kernels, geometry, and discrete data}
\label{sec:kernels_geometry_discrete}

The previous section shows that kernel-gradient drifting is identifiable when the smoothing kernel is characteristic. This turns kernel choice into a central design decision. In Euclidean space, many familiar kernels are characteristic. On curved manifolds, however, kernel construction is more delicate, and it has been approached in many different ways depending on the space on which the kernel is defined: in the ambient space \cite{ozakin2009submanifold,BaePolonik2026KernelSmoothing}, in the Reproducing Kernel Hilbert Space induced by the manifold heat kernel \cite{Caseiro2012SemiIntrinsicMeanShift}, in the tangent space and then projected back to the manifold with the exponential map \cite{KimPark2013GeometricStructures, WuWu2022StrongUniformConsistency}, or in the manifold itself, either through geodesic distances \cite{pelletier2005kernel} or spectral kernels \cite{Cleanthous2020KernelWaveletDensity}. This section focuses on spectral kernels and uses our gradient formulation to extend drifting to Riemannian manifolds and categorical data.

\subsection{Which kernels are valid?}
\label{sec:valid_kernels}
In Euclidean space, translation-invariant kernels with non-vanishing Fourier
transform are characteristic. On manifolds, the situation is more subtle. A tempting construction is the family of exponential radial geodesic kernels
\[
k_g(x,y)
=
\exp\!\left(
-\tau d_g(x,y)
\right),
\qquad
\tau>0.
\]
Although this resembles the Euclidean Gaussian kernel, it is not positive
definite in general on curved manifolds. In fact, these kernels are positive definite for all temperatures if and only if the space is flat~\citep{feragen2015geodesic,
jayasumana2015kernel,steinert2025universal}. Thus, geodesic radial kernels can be useful in practice, especially when the distance is known in closed form and cheap to compute, but they do not provide the identifiability guarantees of Proposition~\ref{prop:identifiability}. This further motivates the gradient reformulation of the drifting field: on general geometries, the original formulation is not identifiable over all temperature values for any geodesic radial kernel.

\subsection{Spectral and Matérn kernels}
\label{sec:spectral_matern_kernels}

Let \((\M,g)\) be a compact Riemannian manifold, and let
\(\{(\lambda_n,\phi_n)\}_{n=0}^{\infty}\) denote the eigenpairs of the
Laplace--Beltrami operator. Spectral kernels are defined by
\begin{equation}
\label{eq:spectral_kernel}
k(x,x')
=
\sum_{n=0}^{\infty}
f(\lambda_n)\,\phi_n(x)\phi_n(x'),
\end{equation}
where \(f : [0,\infty)\to \mathbb{R}_{+}\) is a spectral density. If
\(f(\lambda_n)>0\) for every eigenvalue, then  the kernel is characteristic on the corresponding function class. A particularly useful family is given by Matérn kernels on manifolds
~\citep{borovitskiy2020matern}. For parameters \(\sigma^2,\tau,\nu>0\), one
can define
\begin{align}
\label{eq:matern_kernel}
k_{\nu}(x,x')
&=
\sigma^2 C_{\nu}
\sum_{n=0}^{\infty}
\left(
\frac{2\nu}{\tau^2}+\lambda_n
\right)^{-\left(\nu+\frac{d}{2}\right)}
\phi_n(x)\phi_n(x'),
\\
\label{eq:heat_kernel}
k_{\infty}(x,x')
&=
\sigma^2 C_{\infty}
\sum_{n=0}^{\infty}
\exp\!\left(
-\frac{\tau^2}{2}\lambda_n
\right)
\phi_n(x)\phi_n(x').
\end{align}
The constants \(C_\nu\) and \(C_\infty\) are normalization constants. The limit \(k_\infty\) corresponds to the heat kernel. Since all
spectral coefficients are strictly positive, these kernels are characteristic
under the assumptions above, and therefore yield identifiable kernel-gradient
drifts. In practice, spectral kernels can be approximated by truncated expansion. On simple manifolds such as spheres, the eigenfunctions are available in closed
form through spherical harmonics. On meshes, graphs, and more general geometries, one can use numerical
eigenfunctions or finite-dimensional feature approximations, as implemented in
packages for geometric kernels such as ~\citep{JMLR:v26:24-1185}. This provides a practical
route to kernel-gradient drifting without relying on geodesic computations. 

\subsection{Intrinsic drifting on Riemannian manifolds}
\label{sec:riemannian_drifting}

The gradient formulation extends naturally to Riemannian manifolds. Let
\((\M,g)\) be a Riemannian manifold, and let
\(f_\theta:\mathbb{R}^k\to\M\) be a generator whose pushforward distribution is
\(q_\theta=(f_\theta)_{\#}\mathcal{N}(0,I)\). Given a kernel
\(k:\M\times\M\to\mathbb{R}_{>0}\), define
\begin{equation}
\label{eq:riemannian_gradient_component}
G_{p}^+(x)
=
\frac{
\mathbb{E}_{y\sim p}\!\left[\nabla_x^{\M} k(x,y)\right]
}{
\mathbb{E}_{y\sim p}\!\left[k(x,y)\right]
} \qquad  G_{q_\theta}^-(x)
=
\frac{
\mathbb{E}_{x'\sim q_\theta}\!\left[\nabla_x^{\M} k(x,x')\right]
}{
\mathbb{E}_{x'\sim q_\theta}\!\left[k(x,x')\right]
},
\end{equation}
where \(\nabla_x^{\M}\) is the Riemannian gradient. The \textit{Riemannian
kernel-gradient drift} is then
\begin{equation}
\label{eq:riemannian_gradient_drift}
V^\nabla_{p,q_\theta}(x)
=
G_{p}^+(x)-G_{q_\theta}^-(x),
\qquad
V^\nabla_{p,q_\theta}(x)\in T_x\M.
\end{equation}
Generated samples are transported using the exponential map,
$ \widetilde{x} = \Exp_x\!\left(
\eta V^\nabla_{p,q_\theta}(x)
\right)$. The corresponding \textit{value} of the stop-gradient objective is then
\begin{equation}
\label{eq:riemannian_loss}
\begin{aligned}
\mathcal{L}_{\M}(\theta)
&=
\mathbb{E}_{\epsilon\sim\mathcal{N}(0,I)}
\left[
d_g\!\left(
f_\theta(\epsilon),
\mathrm{sg}\!\left[
\Exp_{f_\theta(\epsilon)}
\!\left(
\eta V^\nabla_{p,q_\theta}(f_\theta(\epsilon))
\right)
\right]
\right)^2
\right]
\\
&=
\mathbb{E}_{\epsilon\sim\mathcal{N}(0,I)}
\left\|
\Log_{f_\theta(\epsilon)}
\!\left(
\mathrm{sg}\!\left[
\Exp_{f_\theta(\epsilon)}
\!\left(
\eta V^\nabla_{p,q_\theta}(f_\theta(\epsilon))
\right)
\right]
\right)
\right\|_g^2
\\
&=
\eta^2
\mathbb{E}_{x\sim q_\theta}
\left\|
V^\nabla_{p,q_\theta}(x)
\right\|_g^2,
\end{aligned}
\end{equation}

where we assume $\eta \|V^\nabla_{p,q_\theta}(x)\|_g < \mathrm{inj}(x)$, with $\mathrm{inj}$ denoting the injectivity radius at point $x$. Note that the Euclidean formulation is recovered by taking \(\M=\mathbb{R}^d\) with the
standard Euclidean metric, for which $\Exp_x(v)=x+v $ and $\Log_x(y)=y-x $.

\subsection{Discrete data via Fisher--Rao simplex geometry}
\label{sec:discrete_simplex}

The Riemannian formulation also gives a route to discrete generation. Following the framework from \citet{Davis_2024_Fisher}, consider the  $d$-dimensional probability simplex
\[
\Delta^d = \left\{ x \in \mathbb{R}^{d+1} \,\middle|\, \mathbf{1}^\top x = 1,\; x \geq 0 \right\}.
\]
Then, a categorical distribution $p(x)$ over $K = d+1$ categories can be represented in $\Delta^d$ by placing a Dirac mass $\delta_i$ with weight $p_i$ at each vertex $i \in \{0, \dots, d\}$. Denote  $
\mathring{\Delta}^d := \left\{ x \in \Delta^d \mid x > 0 \right\}$ for the relative interior of the simplex. We can endow this manifold with the Fisher-Rao metric to obtain a \textit{statistical manifold}, such that there exists an isometric map $\phi : \mathring{\Delta}^d \to S_+^d,$ from the simplex to the positive orthant of the hypersphere. Thus, to generate discrete data, one can just implement the Riemannian drifting model on the positive orthant of the sphere. For more details on this construction, refer to Appendix \ref{app: discrete generation}.

\section{Experiments}
\label{sec:experiments}

\textbf{Goal of the experiments.} Our experiments evaluate three aspects of kernel-gradient drifting. First, we test whether the gradient direction improves synthetic data generation and what is the effect of the Matérn smoothness parameter \(\nu\). Second, we assess the method on non-Euclidean real-world data under different kernel and geometry choices. Finally, we assess whether the Fisher--Rao geometry on the probability simplex provides a practical route to one-step discrete generation.

\textbf{Synthetic experiments.} We use controlled synthetic experiments to isolate the effect of drift direction and geometry. \emph{Displacement} versus \emph{Gradient} compares displacement-based and kernel-gradient drift, while \emph{Euclidean} versus \emph{Manifold} compares ambient updates with intrinsic manifold updates. We evaluate checkerboard and swissroll targets on \(\mathbb{S}^2\) and \(\mathbb{H}^2\), projecting samples to the corresponding 2D charts for metric computation. Note that for the Euclidean formulations we project the samples to the manifold before evaluation.

Table~\ref{tab:toy-table} shows that the gradient formulation improves over the displacement formulation in almost all controlled comparisons. The manifold formulation gives further gains in some cases, especially on checkerboard targets, though less consistently than the gradient correction.
\begin{table*}[h!]
  \centering
  \small
  \setlength{\tabcolsep}{4pt}
  \renewcommand{\arraystretch}{1.12}
  \caption{Drift quality on checkerboard and swissroll targets across spherical and hyperbolic 2D geometries using a Laplace kernel. Results reported as mean $\pm$ standard deviation over 3 seeds. \textit{SW} denotes sliced Wasserstein-2 distance, and \textit{Tile} is black-square accuracy on checkerboard targets.}
  \vspace{7pt}
  \begin{tabular}{@{}lllcccc@{}}
  \toprule
  \textbf{Dataset} & \textbf{Manifold} & \textbf{Metric} &
  \multicolumn{2}{c}{\textbf{Euclidean drift}} &
  \multicolumn{2}{c}{\textbf{Manifold drift}} \\
  \cmidrule(lr){4-5}\cmidrule(l){6-7}
  & & & \textbf{Displacement} & \textbf{Gradient (ours)} & \textbf{Displacement} & \textbf{Gradient (ours)} \\
  \midrule

  \multirow{4}{*}{Checkerboard}
    & \multirow{2}{*}{Sphere}
    & SW $\downarrow$
    & {$0.031 \pm 0.009$}
    & \second{$0.012 \pm 0.000$}
    & $0.032 \pm 0.008$
    & \best{$\mathbf{0.011 \pm 0.002}$} \\

    & 
    & Tile $\uparrow$
    & {$0.76 \pm 0.03$}
    & \second{$0.86 \pm 0.01$}
    & {$0.80 \pm 0.04$}
    & \best{$\mathbf{0.87 \pm 0.01}$} \\

  \addlinespace[2pt]

    & \multirow{2}{*}{Hyperboloid}
    & SW $\downarrow$
    & \second{$0.026 \pm 0.003$}
    & \best{$\mathbf{0.025 \pm 0.001}$}
    & $0.063 \pm 0.022$
    & {$0.037 \pm 0.013$} \\

    &
    & Tile $\uparrow$
    & $0.79 \pm 0.01$
    & \second{$0.85 \pm 0.02$}
    & {$0.80 \pm 0.02$}
    & \best{$\mathbf{0.90 \pm 0.01}$} \\

  \midrule

  \multirow{2}{*}{Swissroll}
    & Sphere
    & SW $\downarrow$
    & \second{$0.018 \pm 0.002$}
    & \best{$\mathbf{0.014 \pm 0.002}$}
    & {$0.026 \pm 0.000$}
    & \best{$\mathbf{0.014 \pm 0.002}$ }\\
    
    & Hyperboloid
    & SW $\downarrow$  
    & \second{$0.041 \pm 0.001$ }
    & \best{$\mathbf{0.030 \pm 0.003}$}
    & {$0.042 \pm 0.007$ }
    & $0.044 \pm 0.015$ \\

  \bottomrule
  \end{tabular}
  
  \label{tab:toy-table}
\end{table*}

\begin{figure}[ht!]
    \centering
    \includegraphics[width=1.00\linewidth]{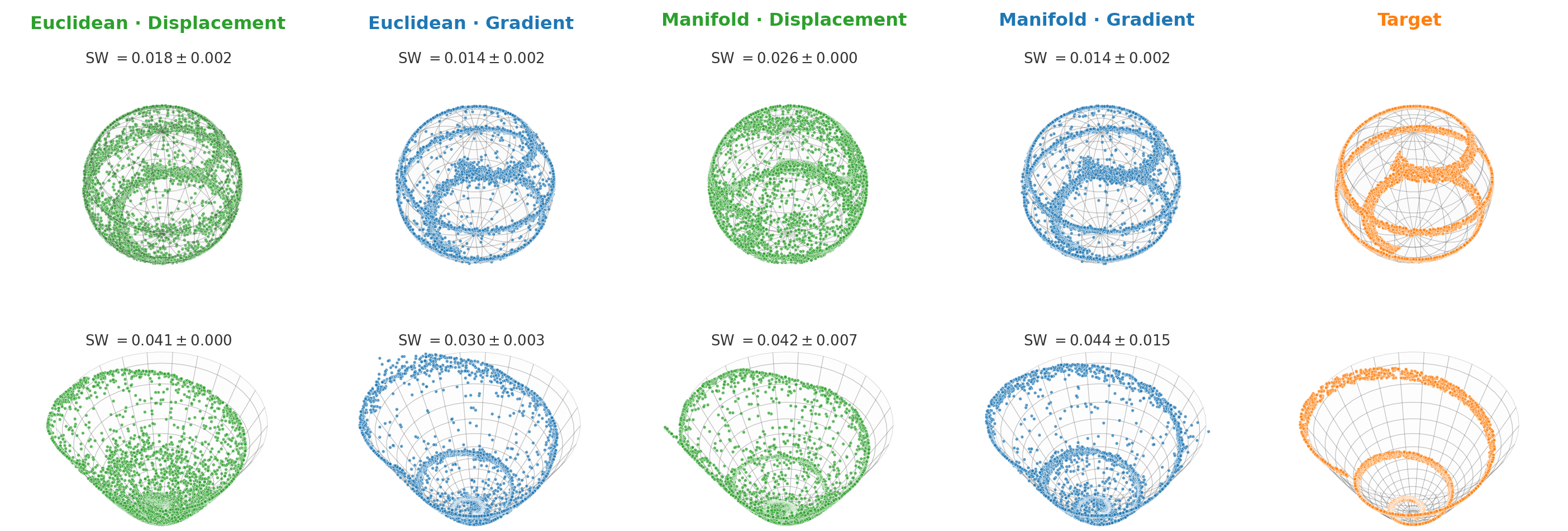}
    \caption{Generated distributions for the swissroll dataset on spherical and hyperboloid manifolds.}
    \label{fig:toy-comparison}
\end{figure}
\newpage
\begin{wrapfigure}{r}{0.50\linewidth}
    \centering
    \vspace{-1em}
    \includegraphics[width=\linewidth]{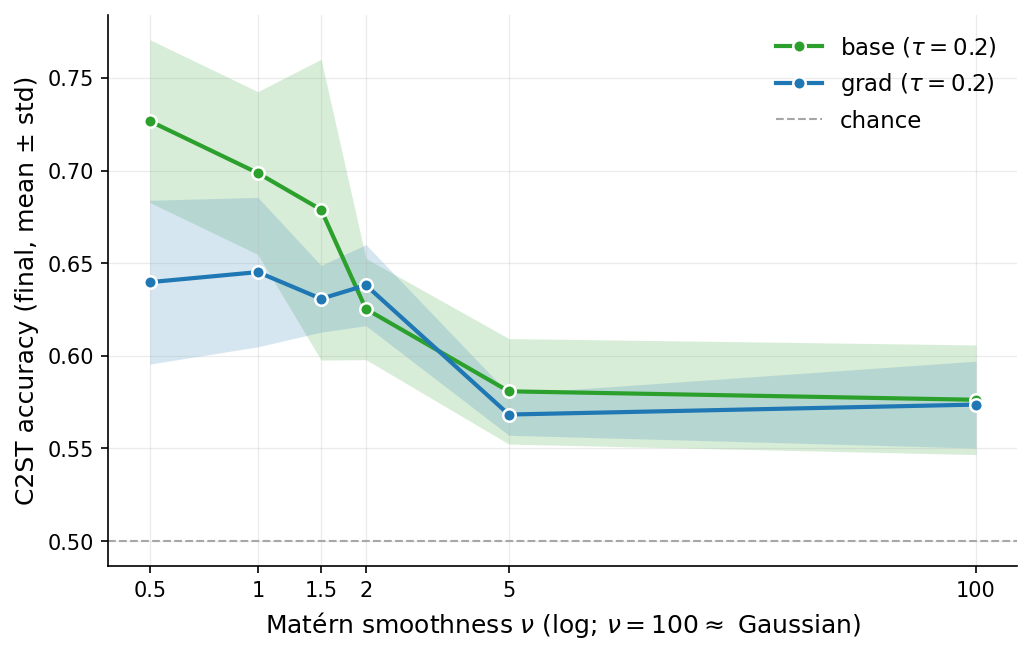}
    \caption{Classifier Two-Sample Test (C2ST) accuracy on the Euclidean Swiss-roll generation task for different Matérn smoothness values \(\nu\). Lower is better, with chance performance indicated at \(0.5\).}
    \label{fig:c2st-vs-nu}
    \vspace{-1em}
\end{wrapfigure}

To assess the effect of kernel choice, we evaluate Euclidean swissroll generation with Matérn kernels of varying smoothness $\nu$. Results are shown in Figure~\ref{fig:c2st-vs-nu}. The sweep confirms that performance of the two methods is equivalent for the Gaussian kernel, but that the gradient direction matters for the non-Gaussian ones, such as Laplace ($\nu$ = 0.5).

\textbf{Geospatial data on the sphere.}
We next evaluate on the geospatial Earth events benchmark on \(\mathbb{S}^2\)~\citep{Mathieu_2020_Earth}. 
This benchmark contains real-world distributions supported on the sphere, and tests whether the geometric formulation is useful beyond controlled synthetic examples. 

Table~\ref{tab:earth_mmd} compares the Euclidean Laplace drifting baseline with our gradient, spherical, and spectral variants. Since Gaussian displacement-based drifting is equivalent to its kernel-gradient form, we report only the Gaussian gradient variants. Here, \emph{Gradient} denotes the kernel-gradient formulation, while \emph{Spherical} uses the intrinsic geometry of \(\mathbb{S}^2\) for distances and updates. 
The results show that our kernel, direction, and geometry choices improve performance on curved real-world data, obtaining the best Maximum Mean Discrepancy (MMD) on all four datasets. Additional metrics are reported in Appendix~\ref{app:additional-results}.

\begin{table*}[ht!]
\centering
\setlength{\tabcolsep}{4pt}
\small
\renewcommand{\arraystretch}{1.1}
\caption{MMD results on the four Earth datasets. Lower is better. The first two row shows the Euclidean drifting baselines from \cite{Deng_2026_Drift}; the remaining rows show our gradient, spherical, and spectral extensions. Best, second-best, and third-best results are highlighted per dataset.}
\begin{tabular}{lcccc}
\toprule
\textbf{Method} & \textbf{Volcano} $\downarrow$ & \textbf{Earthquake} $\downarrow$ & \textbf{Fire} $\downarrow$ & \textbf{Flood} $\downarrow$ \\
\midrule

\quad Euclidean Laplace Displacement {\tiny \cite{Deng_2026_Drift}}
& 0.146 & 0.044 & \third{0.036} & \third{0.064} \\

\quad Euclidean Gaussian Gradient {\tiny \cite{Deng_2026_Drift}}
& 0.143 & 0.047 & 0.048 & 0.072 \\

\midrule

\quad Euclidean Laplace Gradient (ours)
& 0.128 & \second{0.038} & \second{0.030} & 0.067 \\

\quad Spherical Laplace Displacement (ours)
& \second{0.113} & \third{0.043} & 0.049 & \third{0.064} \\

\quad Spherical Gaussian Gradient (ours)
& \best{\textbf{0.112}} & 0.158 & 0.039 & \second{0.058} \\

\quad Spherical Laplace Gradient (ours)
& \second{0.113} & 0.056 & 0.047 & \best{\textbf{0.053}} \\

\quad Spectral ($\nu = 2.5$)  (ours)
& \third{0.126} & \best{\textbf{0.037}} & \best{\textbf{0.029}} & 0.070 \\

\bottomrule
\end{tabular}
\label{tab:earth_mmd}
\end{table*}

\textbf{Discrete data generation.} Finally, we evaluate discrete generation on promoter DNA and QM9. For DNA, we generate sequences over the 4 nucleotides in three geometries: the ambient Euclidean space, the probability simplex (Euclidean updates followed by projection onto the simplex), and the sphere via the Fisher--Rao embedding of Section~\ref{sec:discrete_simplex}; we report 6-mer correlation and Sei MSE. For QM9, we follow \citet{park2026pairflow}, generating SMILES strings and reporting validity and uniqueness. We compare against one-step flow baselines: GFM self-distillation variants  \citep{Davis_2026_GFM} and Riemannian MeanFlow \citep{Woo_2026_RMeanFlow} for DNA, and UDLM and PairFlow \citep{park2026pairflow} for QM9, alongside the original drift \citep{Deng_2026_Drift}. See Appendix~\ref{app:exp-setup} for setup details. The results show that kernel-gradient drifting is a viable teacher-free approach to one-step categorical generation. On promoter DNA (Table~\ref{tab:dna_kmer_corr}), the gradient direction improves over displacement in every geometry, and geometry proves essential: ambient variants reach competitive Sei MSE but much lower 6-mer correlation, indicating that local sequence structure is lost. The spherical gradient model is the strongest geometry-aware variant, on par with the GFM baselines despite simpler training. On QM9 (Table~\ref{tab:qm9_onestep}), our spherical gradient formulation substantially outperforms the Euclidean baseline and approaches the strongest distilled baselines, despite being teacher-free.

\newpage
\begin{wraptable}{R}{0.52\linewidth}
\vspace{-16pt}
\centering
\footnotesize
\setlength{\tabcolsep}{3pt}
\renewcommand{\arraystretch}{1.0}
\caption{Promoter DNA generation results using 6-mer correlation and Sei MSE. We report only the 1-NFE results from previous flow-based methods.}
\vspace{7pt}
\label{tab:dna_kmer_corr}
\begin{tabular}{@{}lcc@{}}
\toprule
\textbf{Method} & \textbf{6-mer} $\uparrow$ & \textbf{Sei MSE} $\downarrow$ \\
\midrule
\multicolumn{3}{@{}l}{\textit{Flow distillation}} \\
\rowcolor{baselinegray}
\quad G-ESD + v-pred {\scriptsize \citep{Davis_2026_GFM}} & $0.13$ & $0.055$ \\
\rowcolor{baselinegray}
\quad G-LSD + v-pred \scriptsize \citep{Davis_2026_GFM} & $0.73$ & $0.046$ \\
\rowcolor{baselinegray}
\quad G-PSD + v-pred \scriptsize \citep{Davis_2026_GFM} & $0.82$ & $0.035$ \\
\rowcolor{baselinegray}
\quad E-RMF + v-pred \scriptsize \citep{Woo_2026_RMeanFlow}& $\mathbf{0.96}\std{0.00}$ & $0.031\std{0.001}$ \\
\rowcolor{baselinegray}
\quad E-RMF + x-pred \scriptsize \citep{Woo_2026_RMeanFlow} & $\mathbf{0.96}\std{0.01}$ & $0.030\std{0.000}$ \\
\rowcolor{baselinegray}
\quad S-RMF + v-pred \scriptsize\citep{Woo_2026_RMeanFlow} & $0.93\std{0.02}$ & $0.030\std{0.001}$ \\
\rowcolor{baselinegray}
\quad S-RMF + x-pred \scriptsize\citep{Woo_2026_RMeanFlow}& $0.84\std{0.03}$  & $0.030\std{0.001}$ \\
\rowcolor{baselinegray}
\quad L-RMF + v-pred \scriptsize\citep{Woo_2026_RMeanFlow}& $0.85\std{0.01}$ & $\mathbf{0.027}\std{0.001}$ \\
\rowcolor{baselinegray}
\quad L-RMF + x-pred \scriptsize\citep{Woo_2026_RMeanFlow} & $0.88\std{0.00}$ & $\mathbf{0.027}\std{0.001}$ \\
\midrule
\multicolumn{3}{@{}l}{\textit{Drifting}} \\
\rowcolor{driftblue}
\quad Euclidean Displacement & $0.56\std{0.03}$ & $0.044\std{0.001}$ \\
\rowcolor{driftblue}
\quad Euclidean Gradient & $0.62\std{0.01}$ & $\mathbf{0.032}\std{0.001}$ \\
\rowcolor{driftblue}
\quad Simplex Displacement & $0.70\std{0.10}$ & $0.047\std{0.004}$ \\
\rowcolor{driftblue}
\quad Simplex Gradient & $0.82\std{0.02}$ & $0.045\std{0.004}$ \\
\rowcolor{driftblue}
\quad Sphere Displacement & $0.78\std{0.08}$ & $0.044\std{0.004}$ \\
\rowcolor{driftblue}
\quad Sphere Gradient & $\mathbf{0.84}\std{0.04}$ & $0.043\std{0.002}$ \\
\bottomrule
\end{tabular}

\vspace{5pt}
\caption{QM9 molecule generation results for one-step methods. Higher is better. Baseline percentages derived from $N=1{,}024$ batch samples.}
\vspace{7pt}
\label{tab:qm9_onestep}
\begin{tabular}{@{}lcc@{}}
\toprule
\textbf{Method} & \textbf{Valid} $\uparrow$ & \textbf{Unique} $\uparrow$ \\
\midrule
\multicolumn{3}{@{}l}{\textit{Flow distillation}} \\
\rowcolor{baselinegray}
\quad UDLM (+DCD) \scriptsize \citep{park2026pairflow}& 31.5 & 31.3 \\
\rowcolor{baselinegray}
\quad PairFlow (+DCD) \scriptsize \citep{park2026pairflow} & \textbf{44.3} & \textbf{44.1} \\
\rowcolor{baselinegray}
\quad UDLM (+ReDi) \scriptsize \citep{park2026pairflow} & 5.8 & 5.8 \\
\rowcolor{baselinegray}
\quad PairFlow (+ReDi)\scriptsize \citep{park2026pairflow}  & 35.3 & 35.1 \\
\midrule
\multicolumn{3}{@{}l}{\textit{Drifting}} \\
\rowcolor{driftblue}
\quad Euclidean Laplace & 22.0 & 40.0 \\
\rowcolor{driftblue}
\quad Spherical Laplace Gradient (ours) & \textbf{38.9} & \textbf{44.1} \\
\bottomrule
\end{tabular}
 \vspace{-25pt}
\end{wraptable}

\newpage

\section{Further Related Work}
\textbf{Riemannian and discrete generative modeling.} Transport-based models have also been adapted to support data on manifolds, and have shown that incorporating non-Euclidean geometry can improve the modeling of these distributions \citep{Chen_2024_RFM, Zaghen2025RiemannianVariationalFlowMatching, Davis_2026_GFM, Woo_2026_RMeanFlow}. In parallel, categorical data generation has been implemented through discrete-time diffusion  \citep{Austin2021StructuredDenoisingDiffusion, Sahoo2024MaskedDiffusionLanguageModels, Shi2024SimplifiedGeneralizedMaskedDiffusion}, continuous-time Markov diffusion \citep{Campbell2022ContinuousTimeDiscreteDenoising, Sun2023ScoreBasedDiscreteDiffusion}, and continuous flow-based relaxations or statistical-manifold formulations  \citep{Eijkelboom_2024, Davis_2024_Fisher, Dieleman2022ContinuousDiffusionCategorical, Cheng_2025_SFM}.

\textbf{Concurrent works on gradient-based drifting models.} The limitations of Gaussian-kernel drifting have also motivated concurrent theoretical work. \citet{Franz_2026_DriftingFields} analyze drifting from an optimization perspective and show that the original \citep{Deng_2026_Drift} style drift is conservative only for the Gaussian kernel. Therefore, it cannot be
written as the gradient of any scalar potential, in general. \citet{Cao_2026_GradientFlowDrifting} derive a kernel density estimation framework that unifies drifting models as gradient flows of divergence functionals. Our work differs in both emphasis and scope. Rather than developing a divergence-level theoretical unification, we focus on aligning drift with the mean-shift direction and evaluate our approach across a range of real-world experiments, whereas their work has not been studied in an experimental setting beyond 2D toy \\ examples.

\section{Conclusion}
\label{sec:discussion_conclusion}
We introduced \emph{kernel-gradient drifting}, a generalization of the original drifting field in which the kernel defines the direction of sample motion through its gradient. This perspective gives a more coherent picture of drifting models: it recovers Gaussian drifting as a special case, extends the score-based interpretation to general kernels, and connects the drifting dynamics to descent of a smoothed Kullback--Leibler divergence. This construction naturally generalizes to data on manifolds and to discrete data generation. Empirically, experiments on synthetic data, spherical Earth event modeling, promoter DNA and molecule generation show that kernel-gradient drifting provides a flexible, teacher-free route to one-step generation. However, there are also some limitations to our work. For instance, our work only focuses on kernel methods, specifically on Matérn kernels, and we found the method to be sensitive to hyperparameter choice. An interesting avenue to explore is to extend the analysis to a broader kernel class and to perform a more systematic study of training stability, which may help close the gap between our method and state-of-the-art. 

\newpage
\section*{Acknowledgements}
This publication is part of the project SIGN (file number VI.Vidi.233.220), which is (partly) financed by the Dutch Research Council (NWO) under grant \url{https://doi.org/10.61686/PKQGZ71565}. JWvdM acknowledges support from the European Union Horizon Framework Programme (Grant agreement ID: 101120237). 

\section*{Ethics statement}
Generative models, can have harmful societal consequences, most notably through the dissemination of disinformation, as well as by amplifying harmful stereotypes and implicit biases. In this work, we aim to advance understanding of drifting models, a specific class of generative models. Although such insights could eventually contribute to improving these models and thereby potentially increase opportunities for misuse, our research does not introduce ethical risks beyond those already associated with generative AI.
\section*{Reproducibility statement}
We include the source code in our submission, which allows for reproducing the results. Our claims made in the main text are proven in the appendices. Experiment details can be found in Appendix \ref{app:exp-setup}.
\section*{Disclosure of LLM Usage}
We have used Large Language Models to polish writing on a sentence level.

\bibliographystyle{plainnat}
\bibliography{refs.bib}

@misc{Lai_2026_Unified,
      title={A Unified View of Drifting and Score-Based Models}, 
      author={Chieh-Hsin Lai and Bac Nguyen and Naoki Murata and Yuhta Takida and Toshimitsu Uesaka and Yuki Mitsufuji and Stefano Ermon and Molei Tao},
      year={2026},
      journal = {arXiv preprint arXiv:2603.07514},
}

@article{Turan_2026_Drifting,
  title={Generative Drifting is Secretly Score Matching: a Spectral and Variational Perspective},
  author={Turan, Erkan and Ovsjanikov, Maks},
  journal={arXiv preprint arXiv:2603.09936},
  year={2026},

}

@article{Deng_2026_Drift,
      title={Generative Modeling via Drifting}, 
      author={Mingyang Deng and He Li and Tianhong Li and Yilun Du and Kaiming He},
      year={2026},
 
    journal = {arXiv preprint arXiv:2602.04770},

}

@article{chen2022sei,
      title={A sequence-based global map of regulatory activity for deciphering human genetics},
      author={Kathleen M. Chen and Aaron K. Wong and Olga G. Troyanskaya and Jian Zhou},
      journal={Nature Genetics},
      volume={54},
      pages={940--949},
      year={2022},
      url={https://doi.org/10.1038/s41588-022-01102-2},
}

@inproceedings{
    Davis_2024_Fisher,
    title={Fisher Flow Matching for Generative Modeling over Discrete Data},
    author={Oscar Davis and Samuel Kessler and Mircea Petrache and Ismail Ilkan Ceylan and Michael M. Bronstein and Joey Bose},
    booktitle={The Thirty-eighth Annual Conference on Neural Information Processing Systems},
    year={2024},
}

@inproceedings{Eijkelboom_2024,
      title= "Variational flow matching for graph generation",
      booktitle = "Advances in Neural Information Processing Systems 37",
      author = "Eijkelboom, Floor and Bartosh, Grigory and Naesseth, Christian and van de Meent, Jan-Willem and Welling, Max",
      year            =  2024,
}

@inproceedings{sonderby2017amortised,
  author    = {S{\o}nderby, Casper Kaae and Caballero, Jose and Theis, Lucas and Shi, Wenzhe and Husz{\'a}r, Ferenc},
  title     = {Amortised {MAP} Inference for Image Super-Resolution},
  booktitle = {Proceedings of the International Conference on Learning Representations},
  year      = {2017}
}

@article{Davis_2026_GFM,
    title={Generalised Flow Maps for Few-Step Generative Modelling on Riemannian Manifolds}, 
    author={Oscar Davis and Michael S. Albergo and Nicholas M. Boffi and Michael M. Bronstein and Avishek Joey Bose},
    year={2025},
    journal = {arXiv preprint arXiv:2510.21608},
}

@article{Yin_2024_DMD,
  title={One-Step Diffusion with Distribution Matching Distillation},
  author={Tianwei Yin and Michael Gharbi and Richard Zhang and Eli Shechtman and Fr{\'e}do Durand and William T. Freeman and Taesung Park},
  journal={Conference on Computer Vision and Pattern Recognition (CVPR)},
  year={2024},
}

@article{Woo_2026_RMeanFlow,
      title={Riemannian MeanFlow}, 
      author={Dongyeop Woo and Marta Skreta and Seonghyun Park and Kirill Neklyudov and Sungsoo Ahn},
      year={2026},
      journal = {arXiv preprint arXiv:2602.07744},
}

@inproceedings{feragen2015geodesic,
author = {Feragen, Aasa and Lauze, Francois and Hauberg, Soren},
year = {2015},
pages = {3032-3042},
title = {Geodesic exponential kernels: When curvature and linearity conflict},
booktitle = {Conference on Computer Vision and Pattern Recognition (CVPR)},
}

@article{JMLR:v26:24-1185,
  author  = {Peter Mostowsky and Vincent Dutordoir and Iskander Azangulov and No{\'e}mie Jaquier and Michael John Hutchinson and Aditya Ravuri and Leonel Rozo and Alexander Terenin and Viacheslav Borovitskiy},
  title   = {The Geometric Kernels Package: Heat and Mat{\'e}rn Kernels for Geometric Learning on Manifolds, Meshes, and Graphs},
  journal = {Journal of Machine Learning Research},
  year    = {2025},
}

@article{Roos_2026_CatFlowMaps,
      title={Categorical Flow Maps}, 
      author={Daan Roos and Oscar Davis and Floor Eijkelboom and Michael Bronstein and Max Welling and İsmail İlkan Ceylan and Luca Ambrogioni and Jan-Willem van de Meent},
      year={2026},
     journal = {arXiv preprint arXiv:2602.12233},
}

@inproceedings{Mathieu_2020_Earth,
 author = {Mathieu, Emile and Nickel, Maximilian},
 booktitle = {Advances in Neural Information Processing Systems},
 title = {Riemannian Continuous Normalizing Flows},
 year = {2020}
}

@book{lee2012smooth,
  title     = {Introduction to Smooth Manifolds},
  author    = {Lee, John M.},
  edition   = {2},
  series    = {Graduate Texts in Mathematics},
  volume    = {218},
  publisher = {Springer},
  year      = {2012},
}

@article{pelletier2005kernel,
  author  = {Bruno Pelletier},
  title   = {Kernel density estimation on Riemannian manifolds},
  journal = {Statistics \& Probability Letters},
  volume  = {73},
  number  = {3},
  pages   = {297--304},
  year    = {2005}
}

@inproceedings{ozakin2009submanifold,
 author = {Ozakin, Arkadas and Gray, Alexander},
 booktitle = {Advances in Neural Information Processing Systems},
 title = {Submanifold density estimation},
 year = {2009}
}

@article{BaePolonik2026KernelSmoothing,
  title   = {Kernel smoothing on manifolds},
  author  = {Bae, Eunseong and Polonik, Wolfgang},
  journal = {arXiv preprint arXiv:2601.16777},
  year    = {2026},
}

@inproceedings{mescheder2018which,
  title     = {Which Training Methods for GANs do Actually Converge?},
  author    = {Mescheder, Lars and Geiger, Andreas and Nowozin, Sebastian},
  booktitle = {Proceedings of the 35th International Conference on Machine Learning},
  series    = {Proceedings of Machine Learning Research},
  year      = {2018},
}

@article{Cleanthous2020KernelWaveletDensity,
  TITLE = {{Kernel and wavelet density estimators on manifolds and more general metric spaces}},
  AUTHOR = {Cleanthous, Galatia and Georgiadis, Athanasios G and Kerkyacharian, Gerard and Petrushev, Pencho and Picard, Dominique},
  JOURNAL = {{Bernoulli Society for Mathematical Statistics and Probability}},
  VOLUME = {26},
  YEAR = {2020},
}

@inproceedings{Caseiro2012SemiIntrinsicMeanShift,
author = {Caseiro, Rui and Henriques, Joao and Martins, Pedro and Batista, Jorge},
year = {2012},
month = {10},
pages = {342-355},
title = {Semi-intrinsic Mean Shift on Riemannian Manifolds},
volume = {7572},
booktitle = {Computer Vision, ECCV},

}

@article{KimPark2013GeometricStructures,
author = {Kim, Yoon Tae and Park, Hyun Suk},
year = {2013},
month = {02},
pages = {112-126},
title = {Geometric structures arising from kernel density estimation on Riemannian manifolds},
volume = {114},
journal = {Journal of Multivariate Analysis},
doi = {10.1016/j.jmva.2012.07.006}
}

@article{WuWu2022StrongUniformConsistency,
author = {Wu, Hau-Tieng and Wu, Nan},
year = {2021},
title = {Strong uniform consistency with rates for kernel density estimators with general kernels on manifolds},
volume = {11},
journal = {Information and Inference: A Journal of the IMA},
}

@inproceedings{borovitskiy2020matern,
 author = {Borovitskiy, Viacheslav and Terenin, Alexander and Mostowsky, Peter and Deisenroth, Marc},
 booktitle = {Advances in Neural Information Processing Systems},
 title = {Mat\'{e}rn Gaussian Processes on Riemannian Manifolds},
 volume = {33},
 year = {2020}
}

@article{jayasumana2015kernel,
author = {Jayasumana, Sadeep and Hartley, Richard and Salzmann, Mathieu and li, Hongdong and Harandi, Mehrtash},
year = {2014},
title = {Kernel Methods on Riemannian Manifolds with Gaussian RBF Kernels},
volume = {37},
journal = {IEEE Transactions on Pattern Analysis and Machine Intelligence},
}

@article{steinert2025universal,
  title={Universal kernels via harmonic analysis on Riemannian symmetric spaces},
  author={Franziskus Steinert and Salem Said and Cyrus Mostajeran},
  journal = {arXiv preprint arXiv:2506.19245},
  year={2025},
}

@article{
weber2023score,
title={The Score-Difference Flow for Implicit Generative Modeling},
author={Romann M. Weber},
journal={Transactions on Machine Learning Research},
year={2023},

}

@inproceedings{gorham2015measuring,
  title={Measuring Sample Quality with Stein's Method},
  author={Gorham, Jackson and Mackey, Lester},
  booktitle={Advances in Neural Information Processing Systems},
  year={2015}
}

@inproceedings{fukumizu2009characteristic,
 author = {Fukumizu, Kenji and Gretton, Arthur and Sch\"{o}lkopf, Bernhard and Sriperumbudur, Bharath K.},
 booktitle = {Advances in Neural Information Processing Systems},
 publisher = {Curran Associates, Inc.},
 title = {Characteristic Kernels on Groups and Semigroups},
 volume = {21},
 year = {2008}
}

@InProceedings{avdeyev2023dirichletdiffusionscoremodel,
  title = 	 {{D}irichlet Diffusion Score Model for Biological Sequence Generation},
  author =       {Avdeyev, Pavel and Shi, Chenlai and Tan, Yuhao and Dudnyk, Kseniia and Zhou, Jian},
  booktitle = 	 {Proceedings of the 40th International Conference on Machine Learning},
  year = 	 {2023},

  publisher =    {PMLR},
}

@article{Cao_2026_GradientFlowDrifting,
      title={Gradient Flow Drifting: Generative Modeling via Wasserstein Gradient Flows of KDE-Approximated Divergences}, 
      author={Jiarui Cao and Zixuan Wei and Yuxin Liu},
      year={2026},
    journal = {arXiv preprint arXiv:2603.10592},
}

@article{Franz_2026_DriftingFields,
      title={Drifting Fields are not Conservative}, 
      author={Leonard Franz and Sebastian Hoffmann and Georg Martius},
      year={2026},
    journal = {arXiv preprint arXiv:2604.06333},

}

@article{
boffi2025flowmap,
    title={Flow map matching with stochastic interpolants: A mathematical framework for consistency models},
    author={Nicholas Matthew Boffi and Michael Samuel Albergo and Eric Vanden-Eijnden},
    journal={Transactions on Machine Learning Research},
    year={2025},
    publisher={TMLR},
}

@inproceedings{
boffi2025howto,
title={How to build a consistency model: Learning flow maps via self-distillation},
author={Nicholas Matthew Boffi and Michael Samuel Albergo and Eric Vanden-Eijnden},
booktitle={The Thirty-ninth Annual Conference on Neural Information Processing Systems},
year={2025},
}

@inproceedings{
Geng2025MeanFlows,
title={Mean Flows for One-step Generative Modeling},
author={Zhengyang Geng and Mingyang Deng and Xingjian Bai and J Zico Kolter and Kaiming He},
booktitle={The Thirty-ninth Annual Conference on Neural Information Processing Systems},
year={2026},
}

@article{Geng2025ImprovedMeanFlows,
author = {Geng, Zhengyang and Lu, Yiyang and Wu, Zongze and Shechtman, Eli and Kolter, J. and He, Kaiming},
year = {2025},
title = {Improved Mean Flows: On the Challenges of Fastforward Generative Models},
  journal = {arXiv preprint arXiv:2512.0201},
}

@inproceedings{
Zhou2025TerminalVelocityMatching,
title={Terminal Velocity Matching},
author={Linqi Zhou and Mathias Parger and Ayaan Haque and Jiaming Song},
booktitle={The Fourteenth International Conference on Learning Representations},
year={2026},

}

@inproceedings{
Lipman_2023,
title={Flow Matching for Generative Modeling},
author={Yaron Lipman and Ricky T. Q. Chen and Heli Ben-Hamu and Maximilian Nickel and Matthew Le},
booktitle={The Eleventh International Conference on Learning Representations },
year={2023},
}

@inproceedings{Ho2020DenoisingDiffusion,
 author = {Ho, Jonathan and Jain, Ajay and Abbeel, Pieter},
 booktitle = {Advances in Neural Information Processing Systems},
 pages = {6840--6851},
 title = {Denoising Diffusion Probabilistic Models},
 volume = {33},
 year = {2020}
}

@inproceedings{
Song2021ScoreBased,
title={Score-Based Generative Modeling through Stochastic Differential Equations},
author={Yang Song and Jascha Sohl-Dickstein and Diederik P Kingma and Abhishek Kumar and Stefano Ermon and Ben Poole},
booktitle={International Conference on Learning Representations},
year={2021},
}

@inproceedings{
Salimans2022ProgressiveDistillation,
title={Progressive Distillation for Fast Sampling of Diffusion Models},
author={Tim Salimans and Jonathan Ho},
booktitle={International Conference on Learning Representations},
year={2022},
}

@article{Luo2023DiffInstruct,
author = {Hu, Tianyang and Li, Zhenguo and Luo, Weijian and Sun, Jiacheng and Zhang, Zhihua and Zhang, Shifeng},
year = {2023},
title = {Diff-Instruct: A Universal Approach for Transferring Knowledge From Pre-trained Diffusion Models},
journal = {arXiv preprint arXiv:2305.18455},

}

@inproceedings{Chen_2024_RFM,
 author = {Chen, Ricky T. Q. and Lipman, Yaron},
 booktitle = {International Conference on Learning Representations},
 title = {Flow Matching on General Geometries},
 year = {2024}
}

@inproceedings{
Zaghen2025RiemannianVariationalFlowMatching,
title={Riemannian Variational Flow Matching for Material and Protein Design},
author={Olga Zaghen and Floor Eijkelboom and Alison Pouplin and Cong Liu and Max Welling and Jan-Willem van de Meent and Erik J Bekkers},
booktitle={The Fourteenth International Conference on Learning Representations},
year={2026},
}

@inproceedings{Austin2021StructuredDenoisingDiffusion,
 author = {Austin, Jacob and Johnson, Daniel D. and Ho, Jonathan and Tarlow, Daniel and van den Berg, Rianne},
 booktitle = {Advances in Neural Information Processing Systems},
 title = {Structured Denoising Diffusion Models in Discrete State-Spaces},
 year = {2021}
}

@inproceedings{
Sahoo2024MaskedDiffusionLanguageModels,
title={Simple and Effective Masked Diffusion Language Models},
author={Subham Sekhar Sahoo and Marianne Arriola and Aaron Gokaslan and Edgar Mariano Marroquin and Alexander M Rush and Yair Schiff and Justin T Chiu and Volodymyr Kuleshov},
booktitle={The Thirty-eighth Annual Conference on Neural Information Processing Systems},
year={2024},
}

@inproceedings{Shi2024SimplifiedGeneralizedMaskedDiffusion,
 author = {Shi, Jiaxin and Han, Kehang and Wang, Zhe and Doucet, Arnaud and Titsias, Michalis},
 booktitle = {Advances in Neural Information Processing Systems},
 title = {Simplified and Generalized Masked Diffusion for Discrete Data},
 year = {2024}
}

@inproceedings{Campbell2022ContinuousTimeDiscreteDenoising,
 author = {Campbell, Andrew and Benton, Joe and De Bortoli, Valentin and Rainforth, Thomas and Deligiannidis, George and Doucet, Arnaud},
 booktitle = {Advances in Neural Information Processing Systems},
 title = {A Continuous Time Framework for Discrete Denoising Models},
 year = {2022}
}

@inproceedings{
Sun2023ScoreBasedDiscreteDiffusion,
title={Score-based Continuous-time Discrete Diffusion Models},
author={Haoran Sun and Lijun Yu and Bo Dai and Dale Schuurmans and Hanjun Dai},
booktitle={The Eleventh International Conference on Learning Representations },
year={2023},
}

@article{Dieleman2022ContinuousDiffusionCategorical,
  title={Continuous diffusion for categorical data},
  author={Sander Dieleman and Laurent Sartran and Arman Roshannai and Nikolay Savinov and Yaroslav Ganin and Pierre H. Richemond and A. Doucet and Robin Strudel and Chris Dyer and Conor Durkan and Curtis Hawthorne and R{\'e}mi Leblond and Will Grathwohl and Jonas Adler},
  year={2022},
  journal = {arXiv preprint arXiv:2211.15089}
}

@inproceedings{ Cheng_2025_SFM,
title={Categorical Flow Matching on Statistical Manifolds},
author={Chaoran Cheng and Jiahan Li and Jian Peng and Ge Liu},
booktitle={The Thirty-eighth Annual Conference on Neural Information Processing Systems},
year={2024}
}

@article{Subbarao2009NonlinearMeanShift,
  author  = {Subbarao, Raghav and Meer, Peter},
  title   = {Nonlinear Mean Shift over Riemannian Manifolds},
  journal = {International Journal of Computer Vision},
  year    = {2009},
}

@book{lee2018riemannian,
  author    = {Lee, John M.},
  title     = {Introduction to Riemannian Manifolds},
  series    = {Graduate Texts in Mathematics},
  volume    = {176},
  edition   = {2},
  publisher = {Springer},
  year      = {2018}
}

@inproceedings{
park2026pairflow,
title={PairFlow: Closed-Form Source-Target Coupling for Few-Step Generation in Discrete Flow Models},
author={Mingue Park and Jisung Hwang and Seungwoo Yoo and Kyeongmin Yeo and Minhyuk Sung},
booktitle={The Fourteenth International Conference on Learning Representations},
year={2026},
}

@article{Wu_2018_MoleculeNet,
  author    = {Wu, Zhenqin and Ramsundar, Bharath and Feinberg, Evan N. and Gomes, Joseph and Geniesse, Caleb and Pappu, Aneesh S. and Leswing, Karl and Pande, Vijay S.},
  title     = {{MoleculeNet}: A Benchmark for Molecular Machine Learning},
  journal   = {Chemical Science},
  volume    = {9},
  number    = {2},
  pages     = {513--530},
  year      = {2018},
  publisher = {Royal Society of Chemistry},
  doi       = {10.1039/C7SC02664A}
}

\newpage
\appendix
\newpage
\section{Notation}

In this section, we report the notations that are used in the paper and the rest of the appendix. For a more in depth introduction to Riemannian geometry refer to \citep{lee2018riemannian} or the Appendix section in \citep{Zaghen2025RiemannianVariationalFlowMatching}.

\begin{tabularx}{\textwidth}{llX}
\toprule
\textbf{Symbol} & \textbf{Name} & \textbf{Description} \\
\midrule
$p$ & data distribution & Target distribution. \\
$q_\theta$ & model distribution & Pushforward distribution induced by the generator $f_\theta$. \\
$f_\theta$ & generator & One-step map from latent noise to samples. \\
$\widehat{p}_k$ & smoothed data density & Kernel-smoothed version of $p$. \\
$\widehat{q}_{\theta,k}$ & smoothed model density & Kernel-smoothed version of $q_\theta$. \\
$\mathrm{sg}(\cdot)$ & stop-gradient & Operator that blocks gradients through its argument for backpropagation. \\
$V_{p,q_\theta}$ & drifting field & Original drift field by \citep{Deng_2026_Drift}. \\
$V^\nabla_{p,q_\theta}$ & kernel-gradient drift & Drift field induced by kernel gradients. \\
$D_{\mathrm{KL}}$ & KL divergence & Kullback--Leibler divergence. \\

$V_p^{+}$ & Original attraction drift & Original attraction drift by \citep{Deng_2026_Drift} that attracts models samples to data distribution $p$.\\ 
$V_{q_\theta}^{-}$ & Original repulsion drift & Original repulsion drift by \citep{Deng_2026_Drift} that  repulse models samples to  from each other.\\ 

$G_p^{+}$ & Gradient attraction drift & Attraction field of model samples towards data distribution $p$.\\ 
$G_{q_\theta}^{-}$ & Gradient repulsion drift & Repulsion field of model samples  from each other.\\ 

$\eta$ & step size & Scale of the drift update. \\
$\tau$ & temperature & Kernel bandwidth. \\
$\nu$ & smoothness & Smoothness parameter of the Matérn kernel. \\
\midrule
$\mathcal{M}$ & manifold & Smooth Riemannian manifold. \\
$g$ & metric & Riemannian metric on $\mathcal{M}$; $\langle \cdot, \cdot \rangle = g_p(\cdot, \cdot)$ and $|\cdot| = \sqrt{g_p(\cdot, \cdot)}$. \\
$T_x\mathcal{M}$ & tangent space & Tangent space at $x\in\mathcal{M}$. \\
$\nabla_x^{\mathcal{M}}$ & Riemannian gradient & Intrinsic gradient with respect to $x$. \\
$\Exp_x$ & exponential map & Maps tangent vectors in $T_x\mathcal{M}$ back to $\mathcal{M}$. Within the injectivity radius, this map is a diffeomorphism. \\
$\Log_x$ & logarithm map & Inverse of $\Exp_x$ locally (within the injectivity radius). Maps points in $\mathcal{M}$ to tangent vectors in $T_x\mathcal{M}$). \\
$d_g(x,y)$ & geodesic distance & Distance induced by the metric $g$. \\
inj($x$) & injectivity radius & The injectivity radius at $x \in \mathcal{M}$ is the supremum of all values
$r > 0$ such that the exponential map from the ball
$B_r(0) \subset T_x\mathcal{M}$ to the manifold $\mathcal{M}$ is injective. \\
\midrule
$\mathbb{S}^2$ & sphere & Unit 2-sphere used for spherical data. \\
$\mathbb{H}^2$ & hyperboloid & Hyperbolic 2D manifold used in synthetic experiments. \\
$\Delta^d$ & simplex & Probability simplex for categorical data. \\
$\mathring{\Delta}^d$ & simplex interior & Relative interior of the probability simplex. \\
$\mathbb{S}^d_+$ & positive sphere orthant & Fisher--Rao embedding space of the simplex. \\
\bottomrule
\end{tabularx}

\newpage
\section{Discrete data generation on the probability simplex} \label{app: discrete generation}
In this section we provide further details on the generation of discrete data in the probability simplex with the Fisher-Rao metric. This construction is based on the work of \citep{Davis_2024_Fisher}. 
Consider the \(d\)-dimensional probability simplex
\[
\Delta^d
=
\left\{
x\in\mathbb{R}^{d+1}
\,\middle|\,
\mathbf{1}^{\top}x=1,\; x\geq 0
\right\},
\]
and its relative interior
\[
\mathring{\Delta}^d
=
\left\{
x\in\Delta^d
\,\middle|\,
x_i>0 \;\; \forall i
\right\}.
\]
A categorical distribution over \(K=d+1\) categories can be represented as a point in \(\Delta^d\), where the vertices of the simplex correspond to hard categories. We can endow this manifold with a Riemannian metric: The Fisher-Rao metric.  Under the convention
\[
g_x(u,v)
=
\sum_{i=0}^{d}
\frac{u_i v_i}{x_i},
\]
the square-root map
\[
\phi(x)
=
2\left(
\sqrt{x_0},\ldots,\sqrt{x_d}
\right)
\]
is an isometry from \(\mathring{\Delta}^d\) to the positive orthant of the sphere of radius \(2\). Equivalently, under a rescaled Fisher--Rao metric, the map \(x\mapsto \sqrt{x}\) isometrically identifies the simplex with the positive orthant of the unit sphere. For a sequence of length \(L\) with \(K\) categories per position, the continuous state space is the product manifold
\[
\left(\mathbb{S}^{K-1}_{+}\right)^L.
\]
The generator outputs points on this product manifold, kernel-gradient drifting is performed using the Riemannian formulation above, and final discrete samples are obtained by projecting each simplex point to a category, for example by taking the nearest vertex or the largest coordinate. Thus the same construction that makes drifting intrinsic on manifolds also provides a principled way to apply one-step drifting to discrete data.

\section{What goes wrong with the original formulation? } \label{app: original formulation}
The argument is very similar to the one presented in \citep{Lai_2026_Unified}. Let's take the original drift formulation (\ref{eq:drifting_field}), generalized to the manifold setting (this just involves replacing the direction $y-x$, by $\Log_xy$). So the drift then looks like 
\[V_{p,q_\theta}(x) =
\frac{
\mathbb{E}_{y \sim p}
\!\left[
k(x,y)\, \Log_{x}(y)
\right]
}{
\mathbb{E}_{y \sim p}
\!\left[
k(x,y)
\right]
}
-
\frac{
\mathbb{E}_{x' \sim q_\theta}
\!\left[
k(x,x')\, \Log_{x}(x')
\right]
}{
\mathbb{E}_{x' \sim q_\theta}
\!\left[
k(x,x')
\right]
},
\qquad
V_{p,q_\theta}(x) \in T_{x}\M, 
\] where $k(x,y)$ is a exponential geodesic radial kernel
\[
k(x, y)
=
\exp\!\left(
\phi\!\left(
\frac{d_g^2(x, y)}{\tau^2}
\right)
\right)
\]for a smooth function $\phi: \mathbb{R} \to \mathbb{R}$. Slightly abusing notation, we write $k\!\left(
\frac{d^2_g(x,y)}{\tau^2}
\right)$ for the geodesic radial kernel $k(x,y)$. The mean-shift direction is then 

\[
\nabla \log \hat p(x)
=
\frac{-2}{\tau^2}
\frac{
\mathbb{E}_{y\sim p}
\!\left[
k\!\left(
\frac{d^2_g(x,y)}{\tau^2}
\right)
\phi'\!\left(
\frac{d^2_g(x,y)}{\tau^2}
\right)
\Log_{x}(y)
\right]
}{
\mathbb{E}_{y\sim p}
\!\left[
k\!\left(
\frac{d^2_g(x,y)}{\tau^2}
\right)
\right]
}
\]

where we have used the identity
\[
\nabla_{x} d^2_g(x,y) = -2\,\Log_{x}(y).
\] 
Note that this construction also follows from the Kernel density estimation (KDE) on Riemannian manifolds \citep{Subbarao2009NonlinearMeanShift}.  Now, we apply the formula 
\[ \mathbb{E}[f(x)g(x)]=\text{Cov}(f(x),g(x))+\mathbb{E}[f(x)]\mathbb{E}[g(x)]\]

to get

\[
\begin{aligned}
\nabla \log \hat p(x)
&=
\frac{1}{\tau^2}
\frac{
\mathbb{E}_{y\sim p}
\!\left[
k\!\left(
\frac{d^2_g(x,y)}{\tau^2}
\right)
\Log_{x}(y)
\right]
}{
\mathbb{E}_{y\sim p}
\!\left[
k\!\left(
\frac{d^2_g(x,y)}{\tau^2}
\right)
\right]
} 
\underbrace{
\mathbb{E}_{y\sim p}
\!\left[
\phi'\!\left(
\frac{d^2_g(x,y)}{\tau^2}
\right)
\right]}_{A_p(x)}
\\
&\quad
+ \frac{1}{\tau^2}\underbrace{\frac{\operatorname{Cov}\!\left(
k\!\left(
\frac{d^2_g(x,y)}{\tau^2}
\right)
\Log_{x}(y),
\phi'\!\left(
\frac{d^2_g(x,y)}{\tau^2}
\right)
\right)}{\mathbb{E}_{y\sim p}
\!\left[
k\!\left(
\frac{d^2_g(x,y)}{\tau^2}
\right)
\right]}.}_{\delta_p(x)}
\end{aligned}
\]

The same decomposition applies to $q_\theta$, and we get 
\[ \nabla\log \frac{\hat p (x)}{\hat q_\theta (x)} = \tau^2\big(V^+_{p}(x) A_p(x)- V^-_{q_\theta}(x) A_{q_\theta}(x) \big) + \delta_p(x) - \delta_{q_\theta}(x).\]

Thus, in the Gaussian case, the drifting field is exactly the mean-shift direction, but this is not necessarily true for non-Gaussian kernels due to the extra factors $A_p, A_{q_\theta}$ and covariance $\delta_p, \delta_{q_\theta}$. The curvature of the manifold enters in through the exponential map. The larger the curvature, the larger is the Riemannian volume distortion, amplifying the mismatch between the mean-shift direction and the true smoothed score direction.

\section{Theoretical guarantees} \label{app: theoretical guarantees}

\begin{lemma}  Consider the gradient-drifting field $V^\nabla_{p,q_\theta}$ as defined in (\ref{eq:gradient_drifting}). Then, 
\[
p = q_\theta \;\;\Rightarrow\;\; V^\nabla_{p,q_\theta}(x) = 0 \quad \text{for all } x.
\]
\end{lemma}

\begin{proof}
The proof follows from \cite{Deng_2026_Drift}. We add it for completeness. Since $k$ is a positive kernel, it is easy to see that $V^\nabla_{p,q_\theta}(x)$ is an antisymmetric drifting field, i.e. $V^\nabla_{p,q_\theta}(x) = -V^\nabla_{q_\theta,p}(x)$ $\text{for all } x$. Then, 
\[ V^\nabla_{p,q_\theta}(x) =  V^\nabla_{q_\theta,p}(x) = - V^\nabla_{p,q_\theta}(x).\]   
\end{proof}

Here we present the formal version of Proposition \ref{prop:score_ratio_matching}.

\textbf{Proposition 3.1} \textbf{(formal)}\label{prop: gaussian_kernel smoothed}
Let $p:\M\to[0,\infty)$ and $q_\theta:\M\to[0,\infty)$ be probability densities with respect to the Riemannian volume measure $d\mathrm{vol}_g$, and consider the normalized kernel  $
k:\M\times \M \to (0,\infty).$ Assume that $\mathcal{M}$ is compact and that given an $x\in \M$ there exists an open neighborhood $U$ of $x$ such that:

\begin{enumerate}
    \item The map $ x \mapsto k(x,y) $
    is $C^1$ on $U$ for almost every $y \in \mathcal{M}.$

    \item For every $x\in U$, the function $
    y \mapsto k(x,y)\,p(y)$ and $
    y \mapsto k(x,y)\,q_\theta(y)$
    are integrable on $\M$.

    \item There exists two integrable functions $
    g \in L^1\!\bigl(p\,d\mathrm{vol}_g\bigr)$ and $
    h \in L^1\!\bigl(q_\theta\,d\mathrm{vol}_g\bigr)$
    such that for all $x\in U$ and for almost every $y$,
    \[
    \|\nabla_{x} k(x,y)\|_g \le g(y)  \hspace{2cm  }\, \, \,  \|\nabla_{x} k(x,y)\|_g \le h(y). 
    \]
\end{enumerate}

Then, if

\[ \hat p(x)=\int_M k(x,y)\,p(y)\,d\mathrm{vol}_g(y),
\, \, \, \, \hat q_\theta(x) = \int_M k(x,y)\,q_\theta(y)\,d\mathrm{vol}_g(y)\]
are the smoothed distributions, the drift operator can be expressed as 
\[V_{p,q_\theta} (x) = \nabla_{x}\log \frac{\hat p (x)}{\hat q_\theta (x)}.\]

\begin{proof} Let
\[
\hat p(x)=\int_M k(x,y)\,p(y)\,d\mathrm{vol}_g(y).
\]

Then, \[
\nabla \log \hat p(x)
= \frac{\displaystyle \nabla_{x} \int_{\M}  k(x,y)\,p(y)\,d\mathrm{vol}_g(y)}
{\displaystyle \int_{\M} k(x,y)\,p(y)\,d\mathrm{vol}_g(y)} =
\frac{\displaystyle \int_{\M} \nabla_{x} k(x,y)\,p(y)\,d\mathrm{vol}_g(y)}
{\displaystyle \int_{\M} k(x,y)\,p(y)\,d\mathrm{vol}_g(y)} 
=
\frac{
\mathbb{E}_{y\sim p}
\!\left[
\nabla_{x}
k\!\left(x,y
\right)
\right]
}{
\mathbb{E}_{y\sim p}
\!\left[
k\!\left(x,y
\right)
\right]
}
\]
Where in the second step we have used the Dominated Convergence Theorem to differentiate under the integral sign. Note that in manifolds, we should use a partition of unity to express the integral as a sum of integrals over $\mathbb{R}^n$, and then differentiate under the
integral signs there (see Proposition 16.33 of \cite{lee2012smooth} for details).  The same procedure is repeated for $q_\theta$ to obtain

\[
\nabla \log \hat q_\theta(x) =
\frac{
\mathbb{E}_{y\sim q_\theta}
\!\left[
\nabla_{x}
k\!\left(x,y
\right)
\right]
}{
\mathbb{E}_{y\sim q_\theta}
\!\left[
k\!\left(x,y
\right)
\right]
}.
\]

Finally, we get

\[V_{ p, q_\theta}(x) =  \nabla\log \frac{\hat p (x)}{\hat q_\theta (x)}.\]
\end{proof}

Note that exponential geodesic kernels satisfy Assumption 2 because they are bounded. However, these kernels are not necessarily smooth outside the cut locus as the distance function is not bijective there. 

Below is the proof of Proposition \ref{prop:score_ratio_matching}.

\textbf{Proposition 3.3 }
\label{prop:identifiability2}
Assume that the kernel \(k\) is characteristic and satisfies the assumptions from Proposition \ref{prop:score_ratio_matching}. If
\[
V^\nabla_{p,q_\theta}(x)=0
\quad \text{for all } x,
\]
then \(p=q_\theta\).

\begin{proof}
By Proposition~\ref{prop:score_ratio_matching}, \(V^\nabla_{p,q}=0\) implies
\[
\nabla_x \log \widehat{p}_k(x)
=
\nabla_x \log \widehat{q}_k(x).
\]
As $k$ is a normalized kernel, this implies
\[
\widehat{p}_k(x)
=
\widehat{q}_k(x).
\]
Since \(k\) is characteristic, this implies \(p=q\).
\end{proof}

Finally we present the proof of Proposition \ref{prop:smoothed_kl_descent}. Note that this proof only holds when the support of the distribution lies in the Euclidean space, as we are assuming the iterative updates are additive. For future work, we would like to extend this proof to manifolds. 

\textbf{Proposition 3.4}  Consider the drift update in the \textit{smoothed} variable \[z_{i+1}=z_i + \eta V_{p_\tau,q_{\theta,\tau}}(z_i),\] where  $z\sim \hat q_{\theta,k}.$  Then, the gradient drift direction gives the steepest infinitesimal decrease of the KL-divergence between the smoothed distributions $\hat p_k (z)$ and $\hat q_{\theta,k} (z)$. This corresponds to the Wasserstein gradient flow of the functional $D_{\mathrm{KL}}(\hat q_{\theta,k} \,\|\, \hat p_k)$.

\begin{proof}
The proof follows from \cite{weber2023score}. Consider the smoothed densities \(p_\tau=p * k_\tau\) and
\(q_{\theta,\tau}=q_\theta * k_\tau\), where \(k_\tau\) is a smoothing kernel.
Equivalently, if \(\xi\sim \mathcal{N}(0,I)\), \(x=f_i(\xi)\sim q_\theta\),
and \(u\sim k_\tau\) is independent noise with density \(k_\tau\), then
\[
z=x+u \sim q_{\theta,\tau}.
\]

Consider the drift update in the smoothed variable $z_{i+1}=z_i +V_{p_\tau,q_{\theta,\tau}}(z_i),$ 
then the KL-divergence between \(q_{\theta,\tau}\) and \(p_\tau\) is
\[
D_{\mathrm{KL}}(q_{\theta,\tau} \,\|\, p_\tau)
= \mathbb{E}_{z \sim q_{\theta,\tau}}
\left[
\log q_{\theta,\tau}(z) - \log p_\tau(z)
\right],
\]

and varies according to its functional derivative,
\[
\nabla_\varepsilon D_{\mathrm{KL}}(q_{\theta,\tau} \,\|\, p_\tau)\big|_{\varepsilon=0}
= - \mathbb{E}_{z \sim q_{\theta,\tau}}
\left[
\operatorname{Tr}(A_{p_\tau} V_{p_\tau,q_{\theta,\tau}}(z))
\right],
\]
where
\[
A_{p_\tau} f(z)
=
\nabla_z \log p_\tau(z)\,
V_{p_\tau,q_{\theta,\tau}}(z)^\top
+
\nabla_z V_{p_\tau,q_{\theta,\tau}}(z)
\]
is the Stein operator \citep{gorham2015measuring}. By applying Stein’s identity, we obtain
\[
\begin{aligned}
\mathbb{E}_{z \sim q_{\theta,\tau}}
\!\left[
\operatorname{Tr}(A_{p_\tau} V_{p_\tau,q_{\theta,\tau}}(z))
\right]
&=
\mathbb{E}_{z \sim q_{\theta,\tau}}
\!\left[
\nabla_z \log p_\tau(z)^\top V_{p_\tau,q_{\theta,\tau}}(z)
\right] \\
&\quad -
\mathbb{E}_{z \sim q_{\theta,\tau}}
\!\left[
\nabla_z \log q_{\theta,\tau}(z)^\top V_{p_\tau,q_{\theta,\tau}}(z)
\right] \\
&=
\mathbb{E}_{z \sim q_{\theta,\tau}}
\!\left[
\left(
\nabla_z \log p_\tau(z)
-
\nabla_z \log q_{\theta,\tau}(z)
\right)^\top
V_{p_\tau,q_{\theta,\tau}}(z)
\right],
\end{aligned}
\]
which is the inner product of the score difference and the flow vector
\(V_{p_\tau,q_{\theta,\tau}}(z)\). Maximizing the reduction in the KL divergence
corresponds to maximizing this inner product. Since the inner product of two
vectors is maximized when they are parallel, choosing
\(V_{p_\tau,q_{\theta,\tau}}(z)\) to output a vector parallel to the score
difference will decrease the KL divergence as fast as possible.
\end{proof}

Note that although $D_{\mathrm{KL}}(q_{\theta,\tau} \,\|\, p_\tau) \leq D_{\mathrm{KL}}(q_{\theta} \,\|\, p)$, we have that $D_{\mathrm{KL}}(q_{\theta,\tau} \,\|\, p_\tau)=0$ if and only if $p=q_\theta$ (as long as the kernel is characteristic). In addition, minimizing the KL divergence of the distributions directly is not usually feasible since it may diverge to infinity of $p$ and $q_\theta$ can have unequal support, so looking at the smoothed distributions is the correct proxy.

\newpage

\section{Additional experimental results}
\label{app:additional-results}
\subsection{Synthetic experiments}
We provide qualitative visualizations of the synthetic experiments to complement the quantitative results in Table~\ref{tab:toy-table}. Figure~\ref{fig:toy-grid} compares the final generated distributions for the four drifting variants: Euclidean base, Euclidean gradient, manifold base, and manifold gradient. The visual results support the trends observed quantitatively. These are, the gradient formulation generally leads to more accurate samples, and incorporating the manifold structure can provide additional improvements in some settings. Figures~\ref{fig:toy-evo-chess-sphere}, \ref{fig:toy-evo-chess-hyperboloid}, \ref{fig:toy-evo-roll-sphere}, and \ref{fig:toy-evo-roll-hyperboloid} provide a more detailed view by showing the evolution of the generated distributions during training at selected checkpoints.

\label{app:additional-results-toy}
\begin{figure}[H]
    \centering
    \includegraphics[width=1\linewidth]{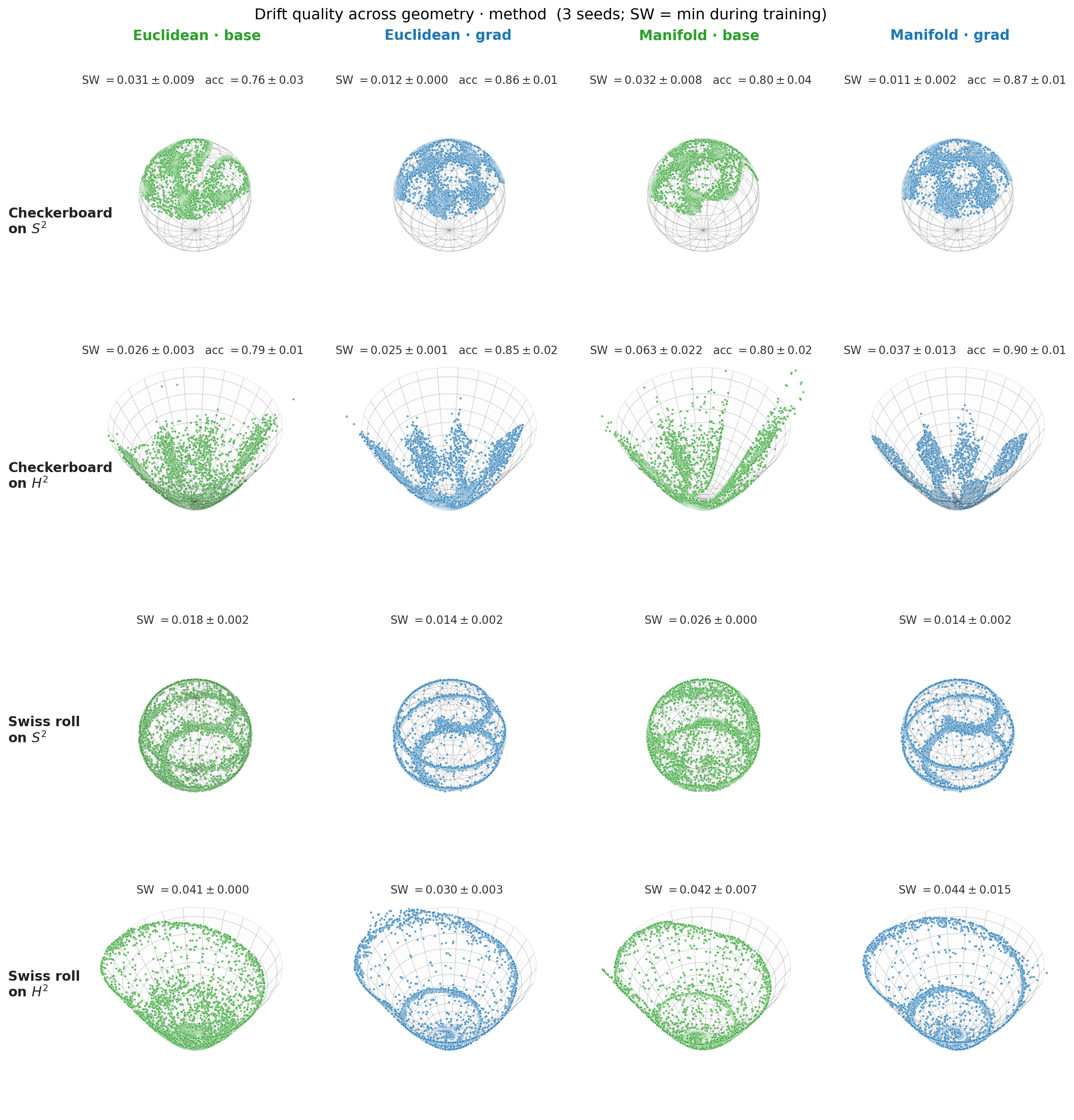}
    \caption{Final generated distributions for the checkerboard and swissroll targets on the sphere $\mathbb{S}^2$ and hyperboloid $\mathbb{H}^2$. We compare the four drifting variants considered in the toy experiments.}
    \label{fig:toy-grid}
\end{figure}

\begin{figure}[H]
    \centering
    \includegraphics[width=1.0\linewidth]{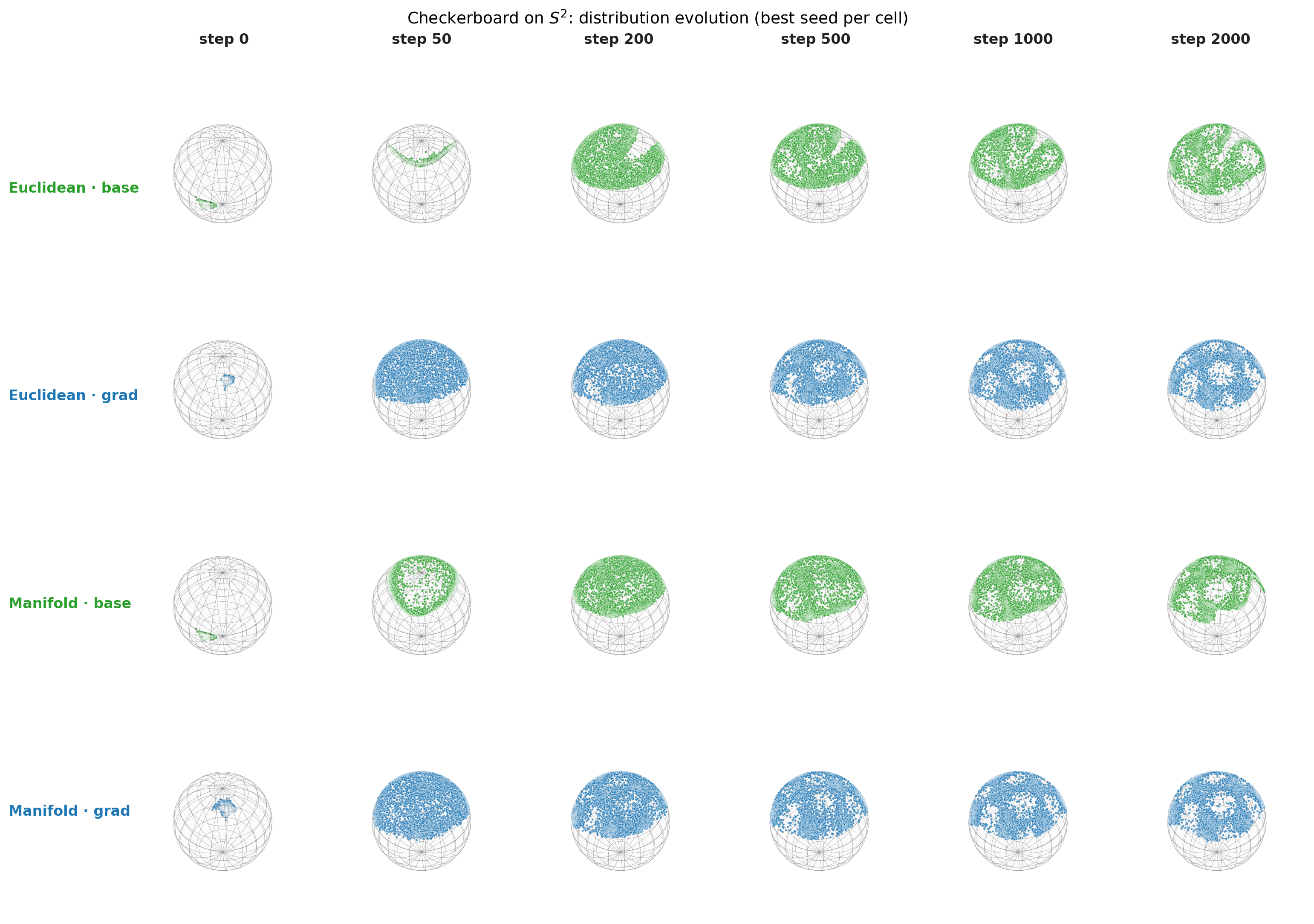}
    \caption{Evolution of the generated distribution during training for the checkerboard target on the spherical manifold $\mathbb{S}^2$.}
    \label{fig:toy-evo-chess-sphere}
\end{figure}

\begin{figure}[H]
    \centering
    \includegraphics[width=0.99\linewidth]{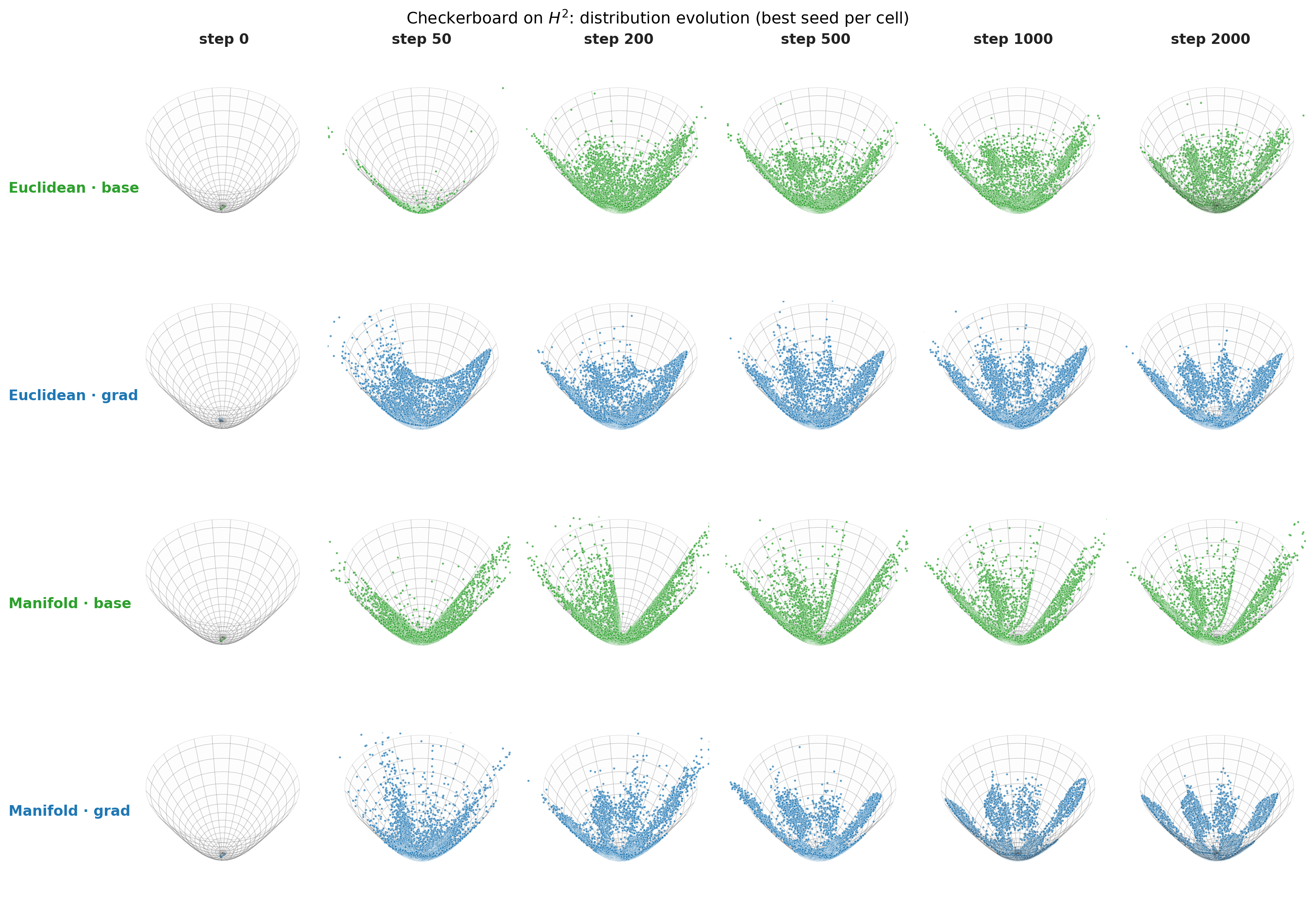}
    \caption{Evolution of the generated distribution during training for the checkerboard target on the hyperbolic manifold $\mathbb{H}^2$.}
    \label{fig:toy-evo-chess-hyperboloid}
\end{figure}
\newpage
\begin{figure}[H]
    \centering
    \includegraphics[width=1.00\linewidth]{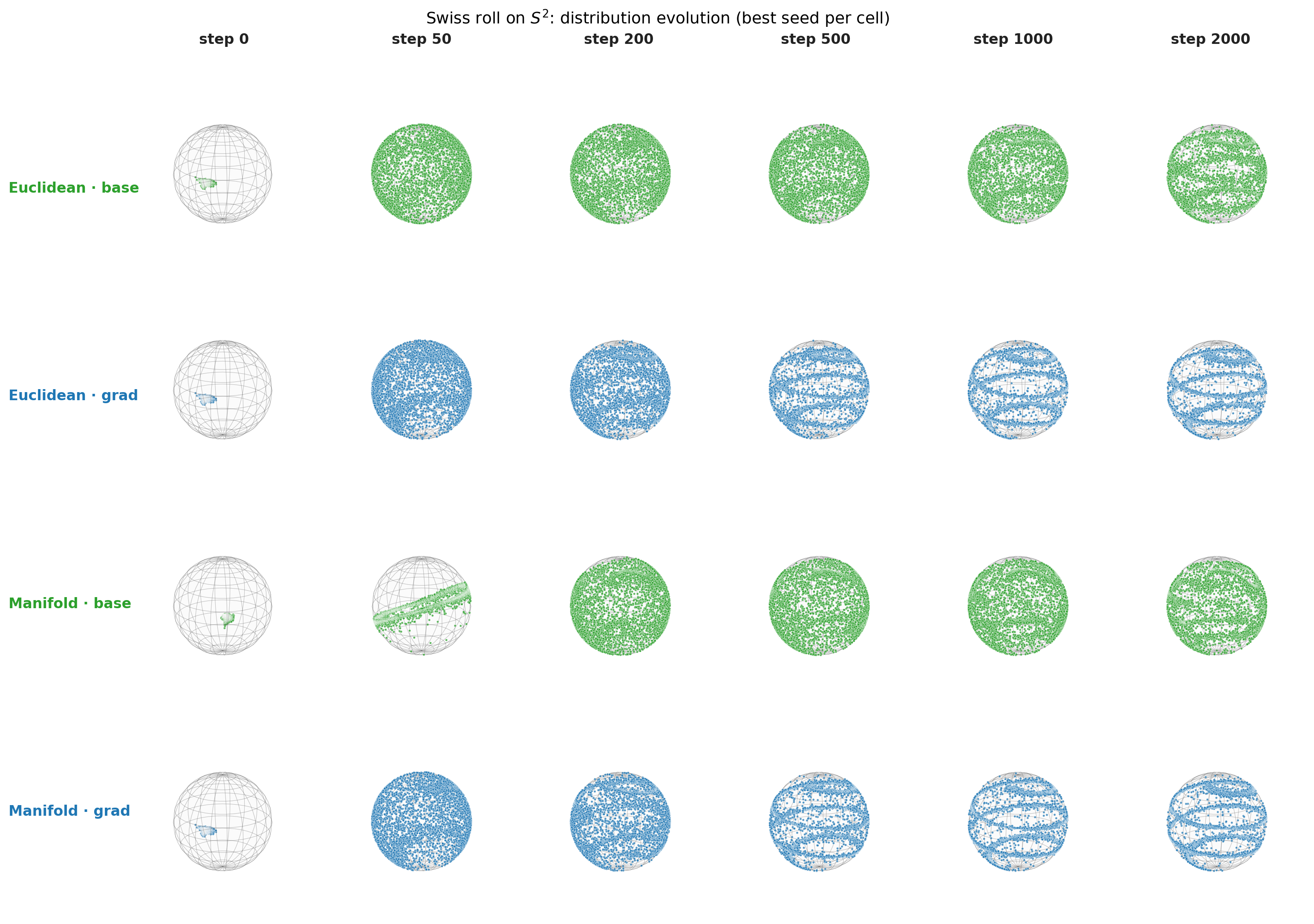}
    \caption{Evolution of the generated distribution during training for the swissroll target on the spherical manifold $\mathbb{S}^2$.}
    \label{fig:toy-evo-roll-sphere}
\end{figure}

\begin{figure}[H]
    \centering
    \includegraphics[width=1.00\linewidth]{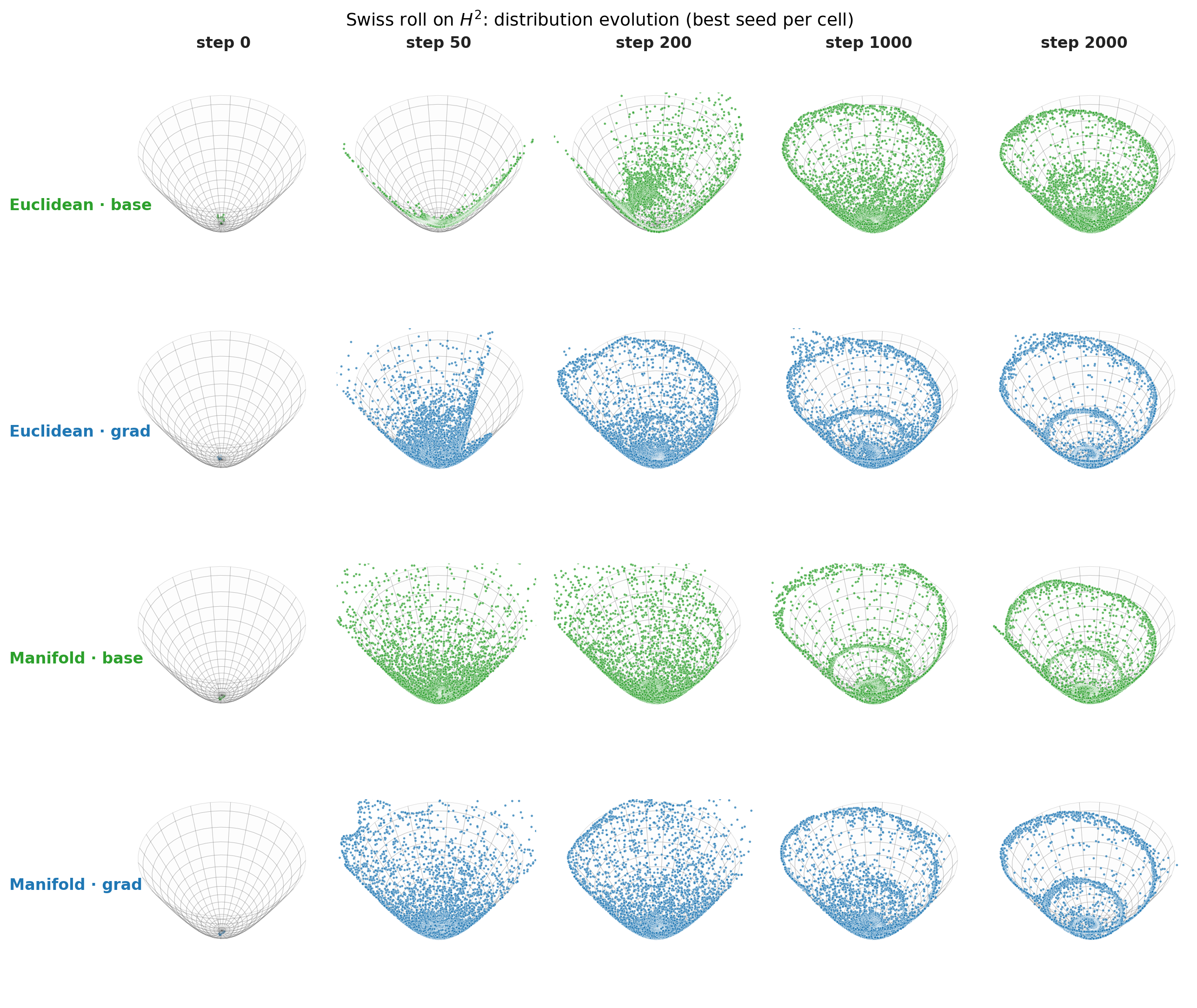}
    \caption{Evolution of the generated distribution during training for the swissroll target on the hyperbolic manifold $\mathbb{H}^2$.}
    \label{fig:toy-evo-roll-hyperboloid}
\end{figure}
\begin{figure}[h!]
\centering
\includegraphics[width=0.9\linewidth]{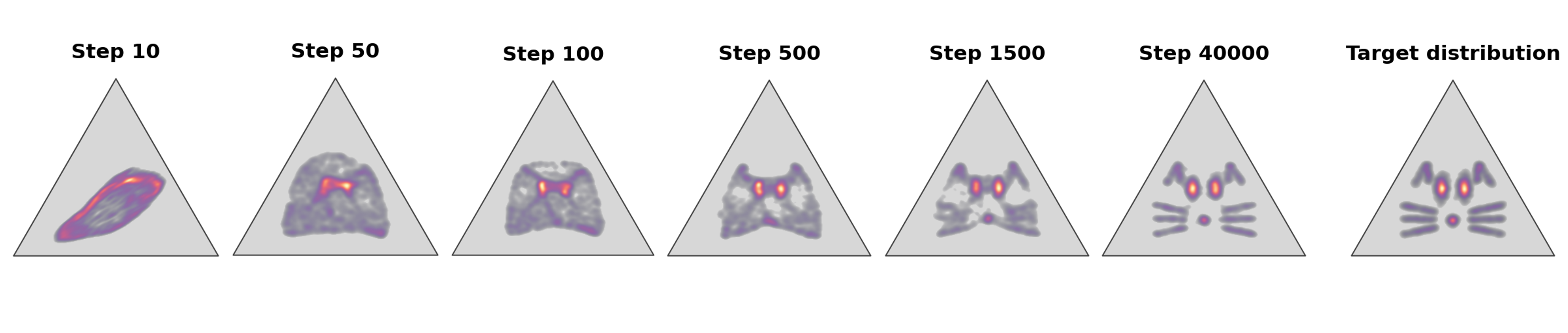}
\caption{Evolution of the 'cat' distribution on the probabiblity simplex during training. }
\label{fig:simplex_spherical_cat}
\end{figure}
\vspace{10pt}

\subsection{Geospatial data on the sphere}
\label{app:additional-results-earth}

In addition to the results reported in Section~\ref{sec:experiments}, we present here a more detailed comparison of the drifting variants. Beyond MMD, we also report geodesic Sinkhorn distance and 1-NN accuracy, as described in Appendix~\ref{app:exp-setup}. Together, these additional metrics provide a broader view of the relative behavior of the different drifting formulations and further illustrate the advantages of the gradient-based approach.

Overall, the results are consistent with the trends observed in the main text. In particular, the gradient-based variants tend to improve not only MMD, but also Sinkhorn distance and 1-NN performance across several datasets. This is especially clear for the Euclidean Laplace Gradient model on \textit{Volcano} and \textit{Earthquake}, and for the spherical gradient variants on \textit{Flood}. On \textit{Fire}, the picture is more mixed, but the gradient formulation still remains competitive across metrics. These results support the conclusion that the proposed reformulation leads to more reliable transport dynamics beyond a single evaluation criterion.

\begin{table*}[h!]
\centering
\small
\setlength{\tabcolsep}{6pt}
\renewcommand{\arraystretch}{1.15}
\begin{tabular}{lccc}
\toprule
\textbf{Method} 
& \multicolumn{3}{c}{\textbf{Volcano}} \\
\cmidrule(lr){2-4} 
& \textbf{MMD} $\downarrow$ & \textbf{Sink.} $\downarrow$ & \textbf{NN1} $\downarrow$\\
\midrule

Euclidean Laplace {\tiny \cite{Deng_2026_Drift}}
& $0.146 \pm 0.022$ & $0.352 \pm 0.035$ & $0.843 \pm 0.044$ \\
\midrule
Euclidean Laplace Gradient
& {$0.128 \pm 0.013$} & \best{$0.283 \pm 0.019$} & \best{$0.711 \pm 0.032$} \\
Euclidean Gaussian Gradient
& $0.143 \pm 0.035$ & $0.334 \pm 0.046$ & \second{$0.766 \pm 0.029$} \\
Spherical Laplace
& \third{$0.113 \pm 0.021$} & $0.320 \pm 0.025$ & $0.801 \pm 0.014$ \\
Spherical Laplace Gradient
& \second{$0.113 \pm 0.014$} & \third{$0.313 \pm 0.029$} & $0.780 \pm 0.026$ \\
Spherical Gaussian Gradient
& \best{$0.112 \pm 0.012$} & {$0.317 \pm 0.015$} & $0.790 \pm 0.018$ \\
Spectral ($\nu = 2.5$)
& {$0.126 \pm 0.022$} & \second{$0.309 \pm 0.031$} & \third{$0.769 \pm 0.017$} \\

\bottomrule
\end{tabular}
\caption{Additional comparison for Volcano dataset.
}
\label{tab:volcano-full}
\end{table*}

\begin{table*}[h!]
\centering
\small
\setlength{\tabcolsep}{6pt}
\renewcommand{\arraystretch}{1.15}
\begin{tabular}{lccc}
\toprule
\textbf{Method} 
& \multicolumn{3}{c}{\textbf{Earthquake}} \\
\cmidrule(lr){2-4} 
& \textbf{MMD} $\downarrow$ & \textbf{Sink.} $\downarrow$ & \textbf{NN1} $\downarrow$\\
\midrule
Euclidean Laplace {\tiny \cite{Deng_2026_Drift}}
& $0.044 \pm 0.013$ & \second{$0.157 \pm 0.014$} & \second{$0.694 \pm 0.016$} \\
\midrule
Euclidean Laplace Gradient
& \second{$0.038 \pm 0.009$} & \best{$0.151 \pm 0.011$} & \best{$0.686 \pm 0.018$}\\
Euclidean Gaussian Gradient
& $0.047 \pm 0.006$ & \third{$0.162 \pm 0.007$} & $0.706 \pm 0.008$ \\
Spherical Laplace
& {$0.043 \pm 0.002$} & $0.179 \pm 0.003$ & $0.732 \pm 0.003$ \\
Spherical Laplace Gradient
& $0.056 \pm 0.008$ & $0.184 \pm 0.009$ & \third{$0.703 \pm 0.005$} \\
Spherical Gaussian Gradient
& $0.158 \pm 0.070$ & $0.326 \pm 0.120$ & $0.707 \pm 0.016$ \\
Spectral ($\nu = 2.5$)
& \best{$0.037 \pm 0.006$} & $0.166 \pm 0.004$ & $0.704 \pm 0.012$ \\
\bottomrule
\end{tabular}
\caption{Additional comparison for Earthquake dataset.
}
\label{tab:earthquake-full}
\end{table*}

\begin{table*}[h!]
\centering
\small
\setlength{\tabcolsep}{6pt}
\renewcommand{\arraystretch}{1.15}
\begin{tabular}{lccc}
\toprule
\textbf{Method} 
& \multicolumn{3}{c}{\textbf{Fire}} \\
\cmidrule(lr){2-4} 
& \textbf{MMD} $\downarrow$ & \textbf{Sink.} $\downarrow$ & \textbf{NN1} $\downarrow$\\
\midrule
Euclidean Laplace {\tiny \cite{Deng_2026_Drift}}
& \third{$0.036 \pm 0.002$} & \second{$0.142 \pm 0.009$} & \third{$0.778 \pm 0.011$} \\
\midrule
Euclidean Laplace Gradient
& \second{$0.030 \pm 0.005$} & \best{$0.129 \pm 0.006$} & \best{$0.749 \pm 0.008$} \\
Euclidean Gaussian Gradient
& $0.048 \pm 0.020$ & {$0.166 \pm 0.015$} & $0.846 \pm 0.001$ \\
Spherical Laplace
& $0.049 \pm 0.006$ & $0.180 \pm 0.007$ & $0.809 \pm 0.011$ \\
Spherical Laplace Gradient
& $0.047 \pm 0.003$ & \third{$0.165 \pm 0.011$} & \second{$0.770 \pm 0.015$} \\
Spherical Gaussian Gradient
& $0.039 \pm 0.003$ & $0.220 \pm 0.003$ & $0.852 \pm 0.002$ \\
Spectral ($\nu = 2.5$)
& \best{$0.029 \pm 0.004$} & $0.170 \pm 0.008$ & $0.825 \pm 0.010$ \\
\bottomrule
\end{tabular}
\caption{Additional comparison for Fire dataset.
}
\label{tab:fire-full}
\end{table*}

\newpage
\begin{table*}[h!]
\centering
\small
\setlength{\tabcolsep}{6pt}
\renewcommand{\arraystretch}{1.15}
\begin{tabular}{lccc}
\toprule
\textbf{Method} 
& \multicolumn{3}{c}{\textbf{Flood}} \\
\cmidrule(lr){2-4} 
& \textbf{MMD} $\downarrow$ & \textbf{Sink.} $\downarrow$ & \textbf{NN1} $\downarrow$\\
\midrule
Euclidean Laplace {\tiny \cite{Deng_2026_Drift}}
& $0.064 \pm 0.010$ & $0.203 \pm 0.019$ & $0.692 \pm 0.021$ \\
\midrule
Euclidean Laplace Gradient
& $0.067 \pm 0.009$ & \third{$0.196 \pm 0.013$}& \best{$0.655 \pm 0.017$} \\
Euclidean Gaussian Gradient
& $0.072 \pm 0.001$ & $0.245 \pm 0.010$ & $0.841 \pm 0.040$ \\
Spherical Laplace
& \third{$0.064 \pm 0.005$} & $0.223 \pm 0.006$ & $0.713 \pm 0.007$ \\
Spherical Laplace Gradient
& \best{$0.053 \pm 0.004$} & \second{$0.193 \pm 0.014$} & \third{$0.672 \pm 0.020$} \\
Spherical Gaussian Gradient
& \second{$0.058 \pm 0.001$} & \best{$0.191 \pm 0.007$} & \second{$0.659 \pm 0.018$} \\
Spectral ($\nu = 2.5$)
& $0.070 \pm 0.009$ & $0.241 \pm 0.006$ & $0.724 \pm 0.010$ \\

\bottomrule
\end{tabular}
\caption{Additional comparison for Flood dataset.
}
\label{tab:flood-full}
\end{table*}

\section{Experimental Setup}
\label{app:exp-setup}

\paragraph{Hardware.}
All experiments are carried out on NVIDIA A100 GPUs.

\subsection{Synthetic experiments}

We consider two complementary synthetic setups: a kernel-smoothness ablation in Euclidean space and a geometry ablation on constant-curvature manifolds. All synthetic models are optimized with Adam using learning rate $10^{-3}$.

\paragraph{Kernel-smoothness ablation.}
For Fig.~\ref{fig:c2st-vs-nu}, we use a 2D swiss-roll target in $\mathbb{R}^2$ and sweep a Matérn kernel over
\[
\nu \in \{0.5, 1, 1.5, 2, 2.5, 5, 100\},
\]
treating $\nu=100$ as a finite approximation to the Gaussian limit. We compare displacement-based drift with our kernel-gradient drift. The backbone is a $(16,16)$ MLP with SiLU activations. We train for $3000$ steps with batch size $256$, kernel bandwidth $T=0.2$, and step cap $\eta_{\max}=1$.

\paragraph{Manifold toy examples.}
For Tab.~\ref{tab:toy-table} and Fig.~\ref{fig:toy-grid}, we use checkerboard and swiss-roll targets on $\mathbb{S}^2$ and $\mathbb{H}^2$. We compare a $2{\times}2$ design crossing
\[
\emph{geometry} \in \{\textsc{Euclidean}, \textsc{Manifold}\}
\qquad \text{and} \qquad
\emph{method} \in \{\textsc{Base}, \textsc{Gradient}\}.
\]
The \textsc{Euclidean} variants use ambient distances and additive updates, while the \textsc{Manifold} variants use geodesic distances and the exponential map. All variants use a (geodesic) Laplace kernel.
The backbone is a width-$256$, depth-$5$ MLP. We train for $2000$ steps with batch size $2048$ and gradient clipping at $1.0$. Full hyperparameter configurations are provided in the accompanying repository.

\paragraph{Metrics.}
For the kernel-smoothness ablation, we report classifier two-sample test (C2ST) accuracy. A small MLP is trained to distinguish real from generated samples; values close to $0.5$ indicate that the two distributions are difficult to distinguish. For the manifold toy examples, we report Sliced Wasserstein-2 distance (lower is better) and, on checkerboard targets, tile accuracy: the fraction of samples that, after being unwrapped to the source 2D chart, fall in a black square of the underlying $4{\times}4$ grid (higher is better). Sliced Wasserstein-2 is computed \emph{extrinsically} on the ambient embeddings ($\mathbb{R}^3$ for $\mathbb{S}^2$, and the Minkowski coordinates treated as $\mathbb{R}^3$ for $\mathbb{H}^2$), using random linear projections rather than an intrinsic manifold variant; since all compared methods produce samples in the same ambient space, this remains a fair head-to-head comparison. All results are averaged over $3$ seeds and reported as mean $\pm$ standard deviation.

\subsection{Geospatial data on the sphere.}

\paragraph{Data.}
We evaluate our method on the Earth benchmark introduced by \citet{Mathieu_2020_Earth}, which consists of geospatial event distributions supported on the sphere $\mathbb{S}^2$. We consider the four standard datasets: \textit{earthquake}, \textit{volcano}, \textit{fire}, and \textit{flood}. Following \citet{Woo_2026_RMeanFlow}, we use a fixed random split with $80\%$ of the data for training, $10\%$ for validation, and $10\%$ for testing.


\paragraph{Optimization.}
We train all models for up to $20$ epochs, or until earlier convergence, using AdamW with $\beta_1=0.9$, $\beta_2=0.95$, and weight decay $10^{-2}$. The learning rate is linearly warmed up during the first $2500$ optimization steps and kept constant afterwards. We maintain an exponential moving average (EMA) of model parameters with decay $0.999$, and report results using the EMA model. Gradients are clipped to norm $1.0$. A single run on an A100 GPU takes approximately $20$ minutes.

\paragraph{Metrics.}
Our primary selection metric is a kernel two-sample statistic on $\mathbb{S}^2$ computed with the geodesic Gaussian kernel $k(x,y) = \exp\!\bigl(-d_{\mathbb{S}^2}(x,y)^2\bigr)$. As discussed in Section~\ref{sec:valid_kernels}, exponential-of-squared-geodesic kernels are not positive-definite on curved manifolds in general, so this statistic is not a strict MMD; we use it as a smooth discrepancy for relative comparison between methods. Empirically, all reported values are non-negative across runs, and rankings are consistent with the geodesic Sinkhorn distance and 1-NN two-sample accuracy, which do not rely on a positive-definite kernel. The Sinkhorn metric uses entropically regularized optimal transport between empirical real and generated distributions, with geodesic cost on the sphere. For 1-NN accuracy, we pool generated and real samples, assign each point the label of its nearest neighbor under geodesic distance, and report the resulting classification accuracy; values closer to $50\%$ indicate that the two distributions are harder to distinguish.

\subsection{Promoter DNA generation}

\paragraph{Data.}
We follow the promoter design benchmark introduced by \citet{avdeyev2023dirichletdiffusionscoremodel}. The task is conditional generation of human promoter sequences of length $L=1024$ over the nucleotide alphabet $\{A,C,G,T\}$. Each sequence is conditioned on a per-position regulatory signal track derived from CAGE transcription start site profiles, denoted \texttt{signal\_1c}. This signal corresponds to the plus-strand CAGE BigWig from the FANTOM CAT release, with strand normalization applied: for negative-strand TSSs, both the sequence and signal are reverse-complemented so that all examples are presented in a consistent orientation.

The dataset consists of the top $100$k most highly expressed transcription start sites, as ranked in the FANTOM CAT annotation. Following \citet{Woo_2026_RMeanFlow}, we use a chromosome-based split: chromosomes $8$ and $9$ for testing, chromosome $10$ for validation, and all remaining autosomes for training.

\paragraph{DNA representation.}
For the Euclidean drift model, ground-truth sequences are represented as one-hot vectors in $\mathbb{R}^{L \times 4}$, while generated samples are relaxed categorical distributions on the product simplex $\Delta_3^L$. For the spherical drift model, we map each categorical distribution to the positive orthant of $\mathbb{S}^3$ using the square-root embedding $p \mapsto \sqrt{p}$. Generation therefore takes place on the product manifold $(\mathbb{S}^3_+)^L$. For the simplex drift model, updates are computed in the ambient space and generated samples are projected back onto the product simplex $\Delta_3^L$ after each step.

\paragraph{Architecture and optimization.}
We use the same dilated 1D-CNN backbone as \citet{Woo_2026_RMeanFlow}, removing time-conditioning components since drifting models do not require a time variable. The network consists of $20$ residual blocks with channel width $256$, GroupNorm, SiLU activations, and dilation schedule $[1,1,4,16,64]\times 4$. This is preceded by a kernel-size $9$ input convolution and followed by two $1{\times}1$ output convolutions that produce per-position logits over the four nucleotides. The same backbone is used for all drifting formulations; geometry-specific operations are applied outside the network to the projected samples. Training uses Adam with an exponential moving average (EMA) of decay $0.999$, and all reported metrics are computed on the EMA model. A single run on an A100 GPU takes approximately one hour.

\paragraph{Metrics.}
We evaluate generated promoters using $6$-mer Pearson correlation, following \citet{avdeyev2023dirichletdiffusionscoremodel}. We aggregate $6$-mer counts over the generated and reference corpora, normalize them to relative frequencies, and compute Pearson correlation over the union of observed $6$-mers. Generated samples are discretized via per-position $\arg\max$ before evaluation. This metric measures whether generated sequences reproduce the local sequence composition of real promoters and serves as our primary model-selection criterion. We additionally report Sei MSE, following \citet{Woo_2026_RMeanFlow}: the mean squared error between the promoter signal profiles predicted by a pretrained Sei model \citep{chen2022sei} for the generated sequences and for the target human promoter sequences sharing the same conditioning signal.

\subsection{Molecular generation on QM9}

\paragraph{Data.}
We evaluate unconditional molecular generation on QM9 ~\citep{Wu_2018_MoleculeNet}. The task is to generate small molecules represented as SMILES strings, which are then parsed into molecular graphs for evaluation. We follow the same preprocessing and train/validation/test split protocol as \citet{park2026pairflow}.

\paragraph{Molecule representation.}
Although molecules are naturally represented as graphs, applying our framework directly to graph representations is complicated by permutation symmetry: a single molecule admits many equivalent node orderings, and the drifting objective would treat these isomorphic copies as distinct points to be pushed apart. Resolving this would require permutation-aware neighborhoods, which we leave to future work. Here, following \citet{park2026pairflow}, we sidestep this difficulty and operate on SMILES strings, viewing each molecule as a categorical sequence. Note that this is a strictly harder setting than most graph-based approaches, which typically condition on the number of atoms; SMILES generation is unconditional on molecule size.

Concretely, each molecule is represented as a SMILES string of fixed length \(N=32\) over a vocabulary \({V}\) of size \(K=40\), padded when necessary. Generation therefore takes place over \({V}^N\). The Euclidean drift model operates on one-hot vectors in \(\mathbb{R}^{N \times K}\), with generated samples relaxed to the product simplex \((\Delta^{K-1})^N\). The spherical drift model maps each per-position categorical distribution to the positive orthant of \(\mathbb{S}^{K-1}\) via the square-root embedding \(p \mapsto \sqrt{p}\), so generation takes place on the product manifold \((\mathbb{S}^{K-1}_{+})^N\).

\paragraph{Architecture and optimization.}
We adopt the graph backbone and training protocol of \citet{park2026pairflow}, adapting the flow-based model to the one-step drifting objective. The network predicts relaxed categorical distributions for SMILES, while geometry-specific operations are applied outside the backbone. At evaluation, generated strings are discretized by taking the per-variable \(\arg\max\).

\paragraph{Metrics.}
We report validity and uniqueness. Validity is the fraction of generated samples that yield a fully sanitizable RDKit molecule that can be converted to SMILES and uniqueness is the fraction of distinct molecules among the valid samples, identified by the canonical SMILES of the largest connected fragment. 

\newpage

\end{document}